%% file: icml2025.tex
\documentclass{article} 

\usepackage{microtype}
\usepackage{graphicx}

\usepackage{booktabs} 

\usepackage{hyperref}
\usepackage[accepted]{icml2025}

\usepackage{amsmath}
\usepackage{amssymb}
\usepackage{mathtools}
\usepackage{amsthm}
\usepackage{colortbl} 

\usepackage{url}

\usepackage{xspace}
\usepackage{dsfont}
\usepackage{afterpage}

\usepackage{enumitem}
\usepackage{booktabs}
\usepackage{caption} 
\usepackage{multirow} 
\usepackage{enumitem}
\usepackage{colortbl} 
\usepackage{bbm}
\usepackage{algorithm}
\usepackage{setspace}                       
\usepackage{lipsum} 
\usepackage{subcaption} 

\usepackage{mathtools} 
\usepackage{bm} 
\usepackage{soul} 
\usepackage{nicefrac}
\usepackage{csquotes}

\usepackage{duckuments}
\usepackage{wrapfig}

\usepackage[capitalize,noabbrev]{cleveref}

\theoremstyle{plain}
\newtheorem{theorem}{Theorem}[section]

\theoremstyle{definition}

\newtheorem{assumption}[theorem]{Assumption}
\theoremstyle{remark}

\newcommand{\bs}{\boldsymbol}

\usepackage{xargs}                      
\newcommandx{\todoll}[2][1=]{\todo[linecolor=blue,backgroundcolor=blue!25,bordercolor=blue,#1]{#2}}

\icmltitlerunning{Differentiable Solver Search for Fast Diffusion Sampling}
\usepackage{xcolor}

\begin{document}
\twocolumn[
\icmltitle{Differentiable Solver Search for Fast Diffusion Sampling}

\begin{icmlauthorlist}
\icmlauthor{Shuai Wang}{nju}
\icmlauthor{Zexian Li}{alibaba}
\icmlauthor{Qipeng Zhang}{alibaba}
\icmlauthor{Tianhui Song}{nju}
\icmlauthor{Xubin Li}{alibaba} \\
\icmlauthor{Tiezheng Ge}{alibaba}
\icmlauthor{Bo Zheng}{alibaba}
\icmlauthor{Limin Wang}{nju,ailab}

\end{icmlauthorlist}
\icmlaffiliation{nju}{State Key Lab of Novel Software Technology, Nanjing University, Nanjing, China.}
\icmlaffiliation{alibaba}{Taobao \& Tmall Group of Alibaba, Hangzhou, China.}
\icmlaffiliation{ailab}{Shanghai AI Lab, Shanghai, China.}
\icmlcorrespondingauthor{Limin Wang}{lmwang@nju.edu.cn}

\begin{center}
    \centering
    \resizebox{1.0\textwidth}{!}{\includegraphics{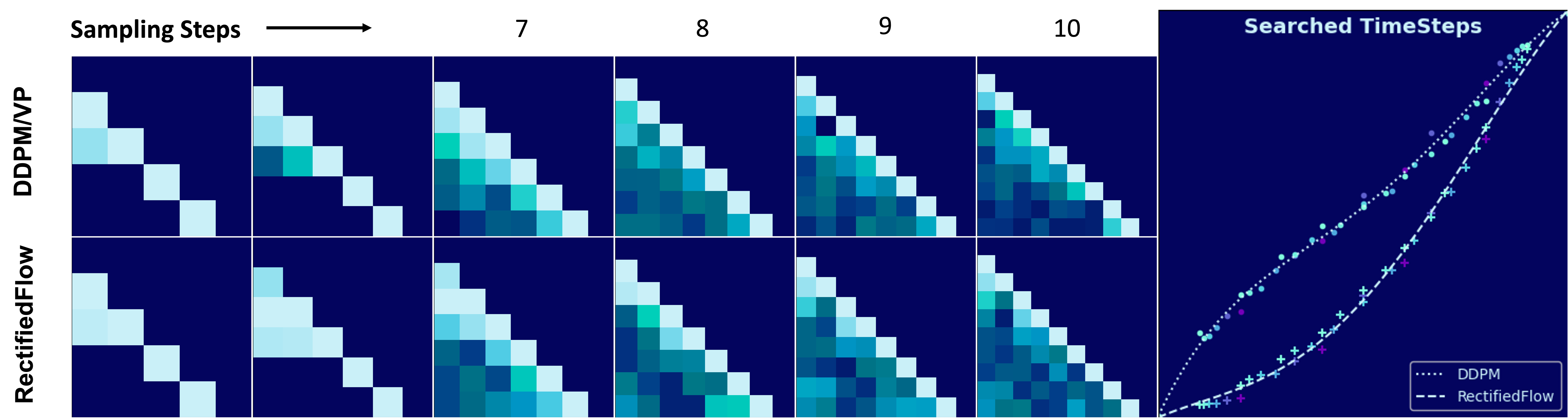}}
    \captionsetup{type=figure}
    \vspace{-2mm}
    \caption{\textbf{Visualization of searched Solver Parameters of DDPM/VP and Rectified Flow. } {\small We limited the order of solver coefficients of the last two steps for 5/6 NFE. The left images show the absolute value of searched coefficients $\{ c_i^j \}$. The right image shows the searched timesteps of different NFE and fitted curves.}}
    \label{fig:vis_solver}
\end{center}
]
\printAffiliationsAndNotice{}
\input{sections/abstract}
\input{sections/introduction}
\input{sections/related}
\input{sections/method2}
\input{sections/experiment}

\input{sections/conclusion}

\section{Limitations} In the main paper, we demonstrate text-to-image visualization with a small CFG value. However, it is intuitive that using a larger CFG would result in superior image quality. We attribute the inferior performance of large CFGs in our solver to the limitations of current naive solver structures and searching techniques. We hypothesize that incorporating predictor-corrector solver structures would enhance numerical stability and yield better images. Additionally, training with CFGs may also be beneficial.

\section*{Impact Statement}
This paper proposes a search-based solver for fast diffusion sampling. We acknowledge that it could lower the barrier for creating diffusion-based AIGC contents.

\paragraph{\bf Acknowledgements.} This work is supported by National Key R$\&$D Program of China (No. 2022ZD0160900), Jiangsu Frontier Technology Research and Development Program (No. BF2024076), and the Collaborative Innovation Center of Novel Software Technology and Industrialization, Alibaba Group through Alibaba Innovative Research Program. 

\bibliography{icml2025}
\bibliographystyle{icml2025}

\appendix
\onecolumn
\newpage
\input{sections/appendix}

\end{document}

%% file: sections/abstract.tex
\begin{abstract}
Diffusion models have demonstrated remarkable generation quality, but at the cost of numerous function evaluations. Advanced ODE-based solvers have recently been developed to mitigate the substantial computational demands of reverse-diffusion solving under limited sampling steps. However, these solvers, heavily inspired by Adams-like multistep methods, rely solely on t-related Lagrange interpolation. We show that t-related Lagrange interpolation is suboptimal for diffusion models and reveal a compact search space comprised of time steps and solver coefficients. Building on our analysis, we propose a novel differentiable solver search algorithm to identify more optimal solver. Equipped with the searched solver, rectified-flow models, e.g., SiT-XL/2 and FlowDCN-XL/2, achieve FID scores of 2.40 and 2.35, respectively, on ImageNet-$256\times256$ with only 10 steps. Meanwhile, DDPM model, DiT-XL/2, reaches a FID score of 2.33 with only 10 steps. Notably, our searched solver significantly outperforms traditional solvers(even some distillation methods). Moreover, our searched solver demonstrates generality across various model architectures, resolutions, and model sizes.
\end{abstract}

%% file: sections/introduction.tex
\section{Introduction}
\vspace{-0.5em}

Image generation is a fundamental task in computer vision research, which aims at capturing the inherent data distribution of original image datasets and generating high-quality synthetic images through distribution sampling. Diffusion models~\cite{ddpm, vp, edm, flow, flow2, decouple_dit} have recently emerged as highly promising solutions to learn the underlying data distribution in image generation, outperforming GAN-based models~\cite{largegan, styleganxl} and Auto-Regressive models~\cite{maskgit} by a significant margin. 

However, diffusion models necessitate numerous denoising steps during inference, which incur a substantial computational cost, thereby limiting the widespread deployment of pre-trained diffusion models. To achieve fast diffusion sampling, the existing studies have explored two distinct approaches. Training-based techniques by distilling the fast ODE trajectory into the model parameters, thereby circumventing redundant refinement steps. In addition, solver-based methods \cite{dpmsolver++, deis, ddim} tackle the fast sampling problem by designing high-order numerical ODE solvers.

For training-based acceleration, \cite{step_distill} aligns the single-step student denoiser with the multi-step teacher output, thereby reducing inference burdens. The consistency model concept, introduced by \cite{cm}, directly teaches the model to produce consistent predictions at any arbitrary timesteps. Building upon \cite{cm}, subsequent works \cite{tcd, ctm, pcm, dmm} propose improved techniques to mitigate discreet errors in LCM training. Furthermore, \cite{lightning, diffusion2gan, dmd, sid, deqdet} leverage adversarial training and distribution matching to enhance the quality of generated samples. To improve the training efficiency of distribution matching. However, training-based methods introduce changes to the model parameters, resulting in an inability to fully exploit the pre-training performance.

Solver-based methods rely heavily on the ODE formulation in the reverse-diffusion dynamics and hand-crafted multi-step solvers. \cite{dpmsolver++, dpmsolver} and \cite{deis}  point out the semi-linear structure of the diffusion ODE and propose an exponential integrator to tackle faster sampling in diffusion models. \cite{unipc} further enhances the sampling quality by borrowing the predictor-corrector structure. Thanks to the multistep-based ODE solver methods, high-quality samples can be generated within as few as 10 steps. To further improve efficiency, \cite{correct} tracks the backward error and determines the adaptive step. Moreover, \cite{edm, dpmsolver} propose a handcrafted timesteps scheduler to sample respaced timesteps. \cite{top} argues that timesteps sampled in \cite{edm, dpmsolver} are suboptimal, thus proposing an online optimization algorithm to find the optimal sampling timesteps for generation. Apart from timesteps optimization, \cite{bespoke} learns a specific path transition to improve the sampling efficiency.

In contrast to training-based acceleration methods, solver-based approaches do not necessitate parameter adjustments and preserve the optimal performance of the pre-trained model. Moreover, solvers can be seamlessly applied to any arbitrary diffusion model trained with a similar noise scheduler, offering a high degree of flexibility and adaptability. This motivates us to investigate the generative capabilities of pre-trained diffusion models within limited steps from a diffusion solver perspective. 

Current state-of-the-art diffusion solvers \cite{dpmsolver++, unipc} adopt Adams-like multi-step methods that use the Lagrange interpolation function to minimize integral errors. We argue that an optimal solver should be tailored to specific pre-trained denoising functions and their corresponding noise schedulers. In this paper, we explore solver-based methods for fast diffusion sampling by improving diffusion solvers using data-driven approaches without destroying the pre-training internality in diffusion models. Inspired by \cite{top}, we investigate the sources of error in the diffusion ODE and discover that the interpolation function form is inconsequential and can be reduced to coefficients. Furthermore, we define a compact search space related to the timesteps and solver coefficients. Therefore, we propose a differentiable solver search method to identify the optimal parameters in the compact search space.

Based on our novel differentiable solver search algorithm, we investigate the upper bound performance of pre-trained diffusion models under limited steps. Our searched solver significantly improves the performance of pre-trained diffusion models, and outperforms traditional solvers with a large gap. On ImageNet-$256\times256$, armed with our solver, rectified-flow models of SiT-XL/2 and FlowDCN-XL/2 achieve 2.40 and 2.35 FID respectively under 10 steps, while DDPM model of DiT-XL/2 achieves 2.33 FID. Surprisingly, our findings reveal that when equipped with an optimized high-order solver, the DDPM can achieve comparable or even surpass the performance of rectified flow models under similar NFE constraints.

To summarize, our contributions are
\begin{itemize}
    \item We reveal that the interpolation function choice is not important and can be reduced to coefficients through the pre-integral technique. We demonstrate that the upper bound of discretization error in reverse-diffusion ODE is related to both timesteps and solver coefficients and define a compact solver search space.
    \item Based on our analysis, we propose a novel differentiable solver search algorithm to find the optimal solver parameter for given diffusion models.
    \item For DDPM/VP time scheduling, armed with our searched solver, DiT-XL/2 achieves 2.33 FID under 10 steps, beating DPMSolver++/UniPC by a significant margin.
    \item For Rectified-flow models, armed with our searched solver, SiT-XL/2 and FlowDCN-XL/2 achieve 2.40 and 2.35 FID respectively under 10 steps on ImageNet-$256\times256$.
    \item For Text-to-Image diffusion models like FLUX, SD3, PixArt-$\Sigma$, our solver searched on ImageNet-$256\times256$ consistently yields better images compared to traditional solvers with the same CFG scale. 
\end{itemize}

%% file: sections/related.tex
\section{Related Work}
\textbf{Diffusion Model} gradually adds ${\bs x_0}$ with Gaussian noise $\epsilon$ to perturb the corresponding known data distribution $p(x_0)$ into a simple Gaussian distribution. The discrete perturbation function of each $t$ satisfies $\mathcal{N}({\bs x}_t|\alpha_t {\bs x}_0, \sigma_t^2 {\bs I})$, where $\alpha_t, \sigma_t > 0$. It can also be written as \cref{eq:ddpm}.
\begin{equation}
    {\bs x}_t = \alpha_t {\bs x}_\text{real} + \sigma_t {\bs \epsilon} \label{eq:ddpm}
\end{equation}
Moreover, as shown in \cref{eq:forward_sde}, \cref{eq:ddpm} has a forward continuous-SDE description, where $f(t) = \frac{\mathrm{d}\log \alpha_t}{\mathrm{d}t}$ and $ g(t) = \frac{\mathrm{d} \sigma_t^2}{\mathrm{d} t} - \frac{\mathrm{d}\log \alpha_t}{\mathrm{d} t}\sigma_t^2$. \cite{reverse_sde} establishes a pivotal theorem that the forward
SDE has an equivalent reverse-time diffusion process as in \cref{eq:reverse_sde}, so the generating process is equivalent to solving the diffusion SDE. Typically, diffusion models employ neural networks and distinct prediction parametrization to estimate the score function $\nabla \log_x p_{{\bs x}_t}({\bs x}_t)$ along the sampling trajectory~\cite{vp, edm, ddpm}.
\begin{align}
     {d}{\bs x}_t &= f(t){\bs x}_t \mathrm{d}t + g(t) \mathrm{d}{\bs w}  \label{eq:forward_sde} \\
     {d}{\bs x}_t &= [f(t){\bs x}_t - g(t)^2\nabla_{\bs x} \log p({\bs x}_t)] dt + g(t) {d}{\bs w}  \label{eq:reverse_sde}
\end{align}

\cite{vp} also shows that there exists a corresponding deterministic process \cref{eq:reverse_ode} whose trajectories share the same marginal probability densities of \cref{eq:reverse_sde}.
\begin{equation}
     {d}{\bs x}_t = [f(t){\bs x}_t - \frac{1}{2}g(t)^2\nabla_{\bs x_t} \log p({\bs x}_t)] {d}t \label{eq:reverse_ode}
\end{equation}
\textbf{Rectified Flow Model} simplifies diffusion model under the framework of \cref{eq:forward_sde} and \cref{eq:reverse_sde}. Different from \cite{ddpm} introduces non-linear transition scheduling, the rectified-flow model adopts linear function to transform data to standard Gaussian noise.
\begin{equation}
    {\bs x}_t = t{\bs x}_\text{real} + (1-t) {\bs \epsilon}\label{eq:flow_matching_forward} \\
\end{equation}
Instead of estimating the score function $\nabla_{\bs x_t} \log p_t({\bs x}_t)$, rectified-flow models directly learn a neural network $v_\theta(x_t, t)$ to predict the velocity field ${\bs v}_t = {d} {\bs x}_t = ({\bs x}_\text{real} - {\bs \epsilon})$.
\begin{equation}
    \mathcal{L}(\theta) = \mathbb{E} [\int_0^1 ||{\bs v}_\theta({\bs x}_t, t) - {\bs v}_t||^2 \mathrm{d}t]
\end{equation}

\textbf{Solver-based Fast Sampling Method} does not necessitate parameter adjustments and preserves the optimal performance of the pre-trained model. It can be seamlessly applied to an arbitrary diffusion model trained with a similar noise scheduler, offering a high degree of flexibility and adaptability. Solvers heavily rely on the reverse diffusion ODE in ~\cref{eq:reverse_ode}. Current solvers are mainly focused on DDPM/VP noise schedules. \cite{dpmsolver, deis} discovered the semi-linear structure in DDPM/VP reverse ODEs. Furthermore, \cite{unipc} enhanced the sampling quality by borrowing the predictor-corrector structure. Thanks to the multi-step ODE solvers, high-quality samples can be generated within as few as 10 steps. To further improve efficiency, \cite{correct} tracks the backward error and determines the adaptive step. Moreover, \cite{edm, dpmsolver} proposed a handcrafted timestep scheduler to sample respaced timesteps. However, \cite{top, ays, gits} argued that the handcrafted timesteps are suboptimal, and thus proposed an online optimization algorithm to find the optimal sampling timestep for generation. Apart from timestep optimization, \cite{bespoke} learned a specific path transition to improve the sampling efficiency.

%% file: sections/method2.tex
\section{Problem Definition}
\begin{figure*}
    \centering
    \includegraphics[width=0.99\linewidth]{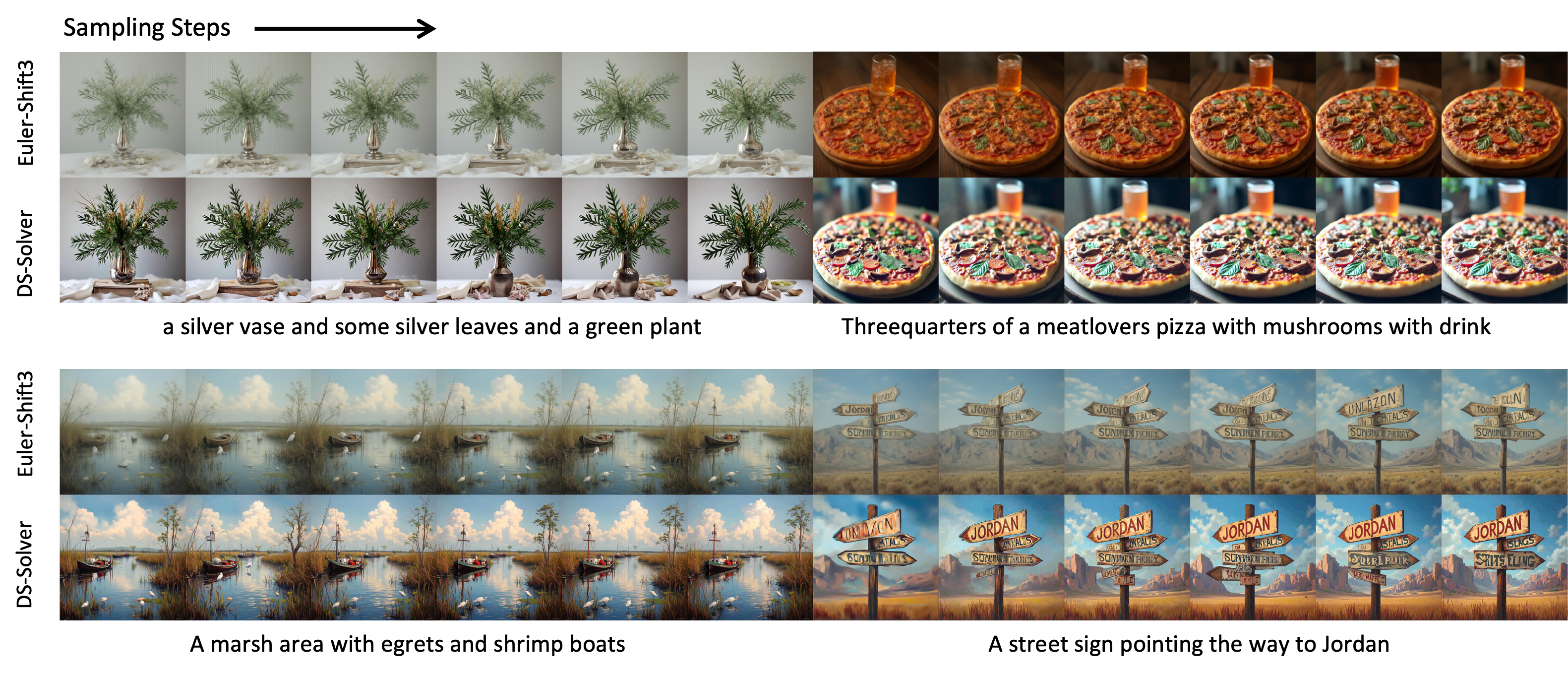}
    \caption{\textbf{Generated images from Flux.1-dev with Guidance=2.0 and our solver~(searched on SiT-XL/2).} {\small Euler-Shift3 is the default solver provided by diffusers and Flux community. Our solver(DS-Solver) achieves better visual quality from 5 to 10 steps(NFE).}}
    \label{fig:vis_flux_cfg2}
    \vspace{-1em}
\end{figure*}
As rectified-flow constitutes a simple yet elegant formulation within the diffusion family, we choose rectified-flow as the primary subject of discussion in this paper to enhance readability. Importantly, our proposed algorithm is not constrained to rectified-flow models. We explore its applicability to other diffusion models such as DDPM/VP in \cref{sec:ddpm_vp_extend}. 

Recall the continuous integration of reverse-diffusion in \cref{eq:fm_split_int} with the pre-defined interval $\{ t_0, t_1, ... t_{N}\}$. Given the pre-trained diffusion models and their corresponding ODE defined in \cref{eq:reverse_ode}, before we tackle the integration of interval $[t_i, t_{i+1}]$, we have already obtained the sampled velocity field of previous timestep $\{ ({\bs x}_j, t_j, {\bs v}_j ={\bs v}_\theta({\bs x}_j, t_j)  \}_{j=0}^i$. Here, we directly denote ${\bs x}_{t_i}$ as ${\bs x}_i$ for presentation clarity:
\begin{equation}
     {\bs x}_{i+1} = {\bs x}_i + \int_{t_i}^{t_{i+1}} {\bs v}_\theta({\bs x}_t, t) {dt} \label{eq:fm_split_int} \\
\end{equation}

As shown in \cref{eq:target}, we strive to develop \textbf{a more optimal solver} that minimizes the integral error while enhancing image quality under limited sampling steps (NFE) without requiring any parameter adjustments for the pre-trained model. 
\begin{align}
    \Phi & = \arg \min \mathbb{E}[|| \Phi({\bs \epsilon}, {\bs v}_\theta) - ({\bs \epsilon} + \int_0^1 {\bs v}_\theta({\bs x}_t, t) {d}t) ||]. \label{eq:target}
\end{align}

\section{Analysis of reverse-diffusion ODE Sampling}
Initially, we revisit the multi-step methods commonly used by \cite{unipc, deis, dpmsolver++} and identify potential limitations. Specifically, we argue that the Lagrange interpolation function used in Adams-Bashforth methods is suboptimal for diffusion models. Moreover, we show that the specific form of the interpolation function is inconsequential, as pre-integration and expectation estimation ultimately reduce it to a set of coefficients. Inspired by \cite{top}, we prove that timesteps and these coefficients effectively constitute our search space.

\subsection{Recap the multi-step methods}
As shown in \cref{eq:fm_split_eular}, the Euler method employs ${\bs v_i}$ as an estimate of \cref{eq:fm_split_eular} throughout the interval $[t_i, t_{i+1}]$. Higher-order multistep solvers further improve the estimation quality of the integral by incorporating interpolation functions and leveraging previously sampled values.
\begin{equation}
     {\bs x}_{i+1} = {\bs x}_i + (t_{i+1} -t_{i}) {\bs v}_\theta({\bs x}_i, t_i). \label{eq:fm_split_eular} \\
\end{equation}
The most classic multi-step solver Adams–Bashforth method~\cite{adam-bashforth}(deemed as Adams for brevity) incorporates the Lagrange polynomial to improve the estimation accuracy within a given interval. 
\begin{align}
    {\bs x}_{i+1} &\approx {\bs x}_i + \int_{t_i}^{t_{i+1}} \sum_{j=0}^i (\prod_{k=0,k\neq j}^i{\frac{t-t_k}{t_j - t_k}}){\bs v}_j dt  \\
     {\bs x}_{i+1} &\approx {\bs x}_i + \sum_{j=0}^i {\bs v}_j \int_{t_i}^{t_{i+1}} (\prod_{k=0,k\neq j}^i{\frac{t-t_k}{t_j - t_k}}) dt \label{eq:fm_lagrange} 
\end{align}
As \cref{eq:fm_lagrange} states, $\int_{t_i}^{t_{i+1}} (\prod_{k=0,k\neq j}^i{\frac{t-t_k}{t_j - t_k}}) dt$ of the Lagrange polynomial can be pre-integrated into a constant coefficient, resulting in only naive summation being required for ODE solving. Current SoTA multi-step solvers~\cite{dpmsolver++, unipc} are heavily inspired by Adams–Bashforth-like multi-step solvers. These solvers employ the Lagrange interpolation function or difference formula to estimate the value in the given interval. 

However, the Lagrange interpolation function and other similar methods only take $t$ into account while the ${\bs v}({\bs x}, t)$ also needs ${\bs x}$ as inputs. Using first-order Taylor expansion of ${\bs x}$ at ${\bs x}_i$ and higher-order expansion of $t$ at $t_i$, we can readily derive the error bound of the estimation. 
\subsection{Focus on Solver coefficients instead of the interpolation function}

In contrast to typical problems of solving ordinary differential equations, when considering reverse-diffusion ODEs along with prerained models, a compact searching space is present. We define a universal interpolation function $\mathcal{P}$, which has no explicit form. $\mathcal{P}$ measures the distance of $({\bs x}_t, t)$ between previous sampled points $\{({\bs x}_j, t_j)\}_{j=0}^i$ to determine the interpolation weight for $\{{\bs v}_j \}_{j=0}^i$. 
\begin{align}
    {\bs x}_{i+1} &\approx {\bs x}_i + \int_{t_i}^{t_{i+1}} \sum_{j=0}^i \mathcal{P}({\bs x}_t, t, {\bs x}_j, t_j){\bs v}_j dt . \\ \label{eq:fm_interpolate} 
    &\approx {\bs x}_i +  \sum_{j=0}^i {\bs v}_j \int_{t_i}^{t_{i+1}} \mathcal{P}({\bs x}_t, t, {\bs x}_j, t_j) dt .
\end{align}

\begin{assumption}
\label{ass:interpolation_bound}
We assume that the remainder term of the universal interpolation function $\sum_{j=0}^i \mathcal{P}({\bs x}_t, t, {\bs x}_j, t_j){\bs v}_j$ for $v({\bs x}, t)$ is bound as $\mathcal{O}(d{\bs x}^m) + \mathcal{O}({dt}^n)$, where $\mathcal{O}(d{\bs x}^m)$ is the $m$th-order infinitesimal for $d{\bs x}$, $\mathcal{O}({dt}^m)$ is the $n$th-order infinitesimal for $dt$.
\end{assumption}

\cref{eq:fm_interpolate} has a recurrent dependency, as ${\bs x}_t$ also relies on $ \sum_{j=0}^i \mathcal{P}({\bs x}_t, t, {\bs x}_j, t_j) {\bs v}_j dt$. To eliminate the recurrent dependency, shown in \cref{eq:fm_interpolate_taylorx}, we simply use the first order Taylor expansion of $x(t)$ at $x_i$ to replace the original form. 
Recall that ${\bs v}_i$ is already determined by ${\bs x}_i$ and $t_i$, thus the partial integral of \cref{eq:fm_interpolate_taylorx} can be formulated as \cref{eq:fm_interpolate_cofun}. Unlike naive Lagrange interpolation, $\mathcal{C}_j({\bs x}_i)$ is a function of the current ${\bs x}_i$ instead of a constant scalar.  Learning a $\mathcal{C}_j({\bs x}_i)$ function will cause the generalization to be lost. This limits the actual usage in diffusion model sampling. 
\begin{align}
    {\bs x}_{i+1} &\approx  {\bs x}_i + \sum_{j=0}^i {\bs v}_j \int_{t_i}^{t_{i+1}}  \mathcal{P}({\bs x}_i + {\bs v}_i(t-t_i), t, {\bs x}_j, t_j) dt \label{eq:fm_interpolate_taylorx} \\
    {\bs x}_{i+1} &\approx {\bs x}_i + \sum_{j=0}^i {\bs v}_j \mathcal{C}_j({\bs x}_i)(t_{i+1} - t_i)  \label{eq:fm_interpolate_cofun}
\end{align}
{
\begin{theorem}
Given sampling time interval $[t_i, t_{i+1}]$ and suppose $\mathcal{C}_j({\bs x_i}) = g_j({\bs x_i}) + b_i^j$, Adams-like linear multi-step methods have an error expectation of $(t_{i+1} - t_i)\mathbb{E}_{{\bs x}_i}||\sum_{j=0}^i {\bs v}_j g_j({\bs x}_i) || $. replacing $\mathcal{C}_j({\bs x})$ with $\mathbb{E}_{{\bs x}_i}[\mathcal{C}_j({\bs x}_i)]$ is the optimal choice and owns an error expectation of $(t_{i+1} - t_i)\mathbb{E}_{{\bs x}_i}||\sum_{j=0}^i {\bs v}_j [g_j({\bs x}_i) -  \mathbb{E}_{{\bs x}_i}g_j({\bs x}_i) || $. We place the proof in \cref{app:proof_pre_error}.
\label{theorom:expcoefficent_error_bound}
\end{theorem}
}

According to \cref{theorom:expcoefficent_error_bound}, we opt to replace $\mathcal{C}_j({\bs x}_i)$ with its expectation $\mathbb{E}_{{\bs x}_i}[\mathcal{C}_j({\bs x}_i)]$, thus we obtain diffusion-scheduler-related coefficients while maintaining generalization ability. 

Finally, given the predefined time intervals, we obtain the optimization target \cref{eq:fm_cofun_target}, where $c_i^j = \mathbb{E}_{{\bs x}_i}[\mathcal{C}_j({\bs x}_i)]$. The expectation can be deemed as optimized through massive data and gradient descent.
\begin{equation}
     {\bs x}_{i+1} \approx {\bs x}_i + \sum_{j=0}^i {\bs v}_j c_i^j (t_{i+1} - t_i)
      \label{eq:fm_cofun_target} 
\end{equation}
\subsection{Optimal search space for a solver}
\begin{assumption}
{As shown in \cref{eq:error_bound}}, the pre-trained velocity model ${\bs v}_\theta$ is not perfect and the error between ${\bs v}_\theta$ and ideal velocity field $\hat {\bs v}$ is L1-bounded, { where $\eta$ is a constant scalar}.
\begin{equation}
    ||\hat{\bs v} - {\bs v}_\theta|| \leq \eta \ll ||\hat{\bs v}||  \label{eq:error_bound}
\end{equation}
\end{assumption}

Previous discussions assume we have a perfect velocity function. However, the ideal velocity is hard to obtain, we only have pre-trained velocity models.  Following \cref{eq:fm_cofun_target}, we can expand \cref{eq:fm_cofun_target} from $t_{i=0}$ to $t_{i=N}$ to obtain the error bound caused by non-ideal velocity estimation. 
{
\begin{theorem}
The error caused by the non-ideal velocity estimation model can be formulated in the following equation. We can employ triangle inequalities to obtain the error-bound(L1) of ${||{\bs x}_N - {\bs \hat{x}}_N||}$, the proof can be found in the \cref{app:proof_total_error}.
$$
{||{\bs x}_N - {\bs \hat{x}}_N||} \leq \eta \sum_{i=0}^{N-1}\sum_{j=0}^i |c_i^j (t_{i+1} - t_{i})|
$$
\label{theorom:non_ideal_error_bound}
\end{theorem}
}
Based on \cref{theorom:non_ideal_error_bound}, since the error bound is related to timesteps and solver coefficients, we can define a much more compact search space consisting of $\{ c_i^j \}_{j < i, j=0, i=1}^N$ and $\{ t_i \}_{i=0}^{N}$. 

{
\begin{theorem}
Based on \cref{theorom:non_ideal_error_bound} and \cref{theorom:expcoefficent_error_bound}. We can derive the total upper error bound(L1) of our solver search method and other counterparts.

The total upper error bound of Our solver search is:

\begin{align*}
    \sum_{i=0}^{N-1} (t_{i+1} - t_{i}) (\sum_{j=0}^i \eta |\mathbb{E}_{{\bs x}_i}g_j({\bs x}_i) + b_i^j| \\
    + \mathbb{E}_{{\bs x}_i}||\sum_{j=0}^i {\bs v}_j g_j({\bs x}_i) - \mathbb{E}_{{\bs x}_i}g_j({\bs x}_i) ||) 
\end{align*}

Compared to Adams-like linear multi-step methods. Our searched solver has a small upper error bound. The proof can be found in the \cref{app:proof_total_error}. 
\label{theorom:total_error_bound}
\end{theorem}
}

Through \cref{theorom:total_error_bound}, our searched solvers own a relatively small upper error bound. Thus we can theoretically guarantee optimal compared to Adams-like methods.

\section{Differentiable solver search.}

Through previous discussion and analysis, we identify $\{ c_i^j \}_{j < i, j=0, i=1}^N$ and $\{ t_i \}_{i=0}^{N}$ as the target search items. To this end, we propose a differentiable data-driven solver search approach to determine these searchable items.

\begin{figure*}
\begin{minipage}{.49\textwidth}
\begin{algorithm}[H]
\caption{Solver Parametrization}
\label{alg:repsolver}
\begin{algorithmic}
\small
    \STATE {\bfseries Requires: }  $\{r_i, \}$ and $\{c_i^j, \}$
    
    \STATE {\bfseries TimeDeltas:} $\Delta t_0, \Delta t_1, ..., \Delta t_{n-1}$.
    
    \STATE {\bfseries SolverCoefficients:} $\mathcal{M} \in R^{N\times N}$
     
    $\{\Delta t_i, \}$=\text{Softmax}($\{r_i\}$)
    
    $\mathcal{M}=\begin{bmatrix}
        1 & & &\\
        c_1^0 & 1-c_1^0 &  &  \\
        \vdots & \vdots & \vdots & \ddots \\
        c_{n-1}^0 & c_{n-1}^1 & \cdots  & 1-\sum_{k=0}^{n-1}c_{n-1}^k \\
    \end{bmatrix}$
\end{algorithmic}
\end{algorithm}
\end{minipage}\hfill
\begin{minipage}{.49\textwidth}
\begin{algorithm}[H]
\caption{Differentiable Solver Search}
\begin{algorithmic}
   \small
   \STATE {\bfseries Require:}  $\bs v_{\theta}$ model, $\{\Delta t_i, \}_{i=0}^{N-1}$, $\mathcal{M}$, A buffer $Q$.
   
   \STATE Compute $\{\tilde{\bs x}_l, \}_{l=0}^L = \textbf{Euler}({\bs \epsilon}, v_\theta)$ .
   \FOR{$i=0$ {\bfseries to} $N-1$}
        \STATE $Q \overset{\mathrm{buffer}}{\leftarrow} \bs v_{\theta}({\bs x}_{t_i}, t_i)$
        \STATE Compute $ {\bs v} = \sum_{j=0}^i \mathcal{M}_{ij}Q_j$.
        \STATE $t_{i+1} = t_{i} + \Delta t_{i}$ \label{eq:deltas2timestep}
        \STATE ${\bs x}_{t_{i+1}} = {\bs x}_{t_i} + {\bs v}\Delta t_{i}$
   \ENDFOR
   \STATE {\bfseries return:} $\tilde{\bs x}_{t_{n-1}}$, $\mathcal{L}(\{\tilde{\bs x}_l \}_{l=0}^L , \{{\bs x}_i \}_{i=0}^{N})$
\end{algorithmic}
\end{algorithm}
\end{minipage}
\end{figure*}

\textbf{Timestep Parametrization.} As shown in \cref{alg:repsolver}, we employ unbounded parameters $\{ r_i,\}_{i=0}^{N-1}$ as the optimization subject, as the integral interval is from 0 to 1, we convert $r_i$ into time-space deltas $\Delta t_i$ with softmax normalization function to force their summation to $1$. We can access timestep $t_{i+1}$ through $t_{i+1} = t_{i} + \Delta t_{i}$. We initialize $\{ r_i\}_{i=0}^{N-1}$ with $1.0$ to obtain a uniform timestep distribution. 

\textbf{Coefficients Parametrization.} Inspired by \cite{top}. Given \cref{eq:fm_cofun_target} and \cref{eq:fm_split_int}, when the velocity field $v_\theta(x, t)$ yields constant value, an implicit constraint $\sum_{k=0}^i c_k^i = 1$ emerges. This observation motivates us to re-parameterize the diagonal value of $M$ as $\{1 - \sum_{j=0}^{i-1} c_i^j, \}_{i=0}^{N-1}$. We initialize $\{c_i^k, \}$ with zeros to mimic the behavior of the Euler solver.

\textbf{Mono-alignment Supervision.} We take the $L$-step Euler solver's ODE trajectory $\{ \tilde{\bs x} \}_{l=0}^L$  as reference. We minimize the gap between the target and source trajectories with the MSE loss. We also adopt Huber loss as auxiliary supervision for ${\bs x}_{t_N}$.

\section{Extending to DDPM/VP framework}
\label{sec:ddpm_vp_extend}

Applying our differentiable solver search to DDPM is infeasible. However, \cite{vp} suggests that there exists a continuous SDE process with $\{f(t) = -{\frac{1}{2}}\beta_t; g(t) = \sqrt{\beta_t} \}$ corresponding to discrete DDPM. This motivates us to transform the search space from the infeasible discrete space to its continuous SDE counterpart.
\cite{dpmsolver} and \cite{deis} discover the semi-linear structure of diffusion and propose exponential integral with $\epsilon$ parametrization to tackle the fast sampling problem of DDPM models, where $\alpha_t = e^{\int_0^t -\frac{1}{2}\beta_s ds}$, $\sigma_t = \sqrt{1- e^{\int_0^t -\beta_s ds}}$ and $\lambda_t = \log \frac{\alpha_t}{\sigma_t}$. \cite{dpmsolver++} further discovers that $x$ parametrization is more powerful for diffusion sampling under limited steps, where $\bar{\bs x} = \frac{{\bs x}_t - \sigma {\bs \epsilon}}{\alpha_t}$.
\begin{equation}
     {\bs x}_t = \frac{\sigma_t} {\sigma_s}{\bs x}_s + \sigma_t \int_{\lambda_s}^{\lambda_t}e^{\lambda}\bar{\bs x}_\theta({\bs x}_{t(\lambda)}, t(\lambda))\mathrm{d}\lambda \label{eq:exp_int_x} 
\end{equation}
We opt to follow the ${\bar x}$ parametrization as DPM-Solver++. However, we find directly interpolating $e^{\lambda}{\bs x}_\theta(x_t, t)$ as a whole part is hard for searching, and yields worse results. To avoid conflating the interpolation coefficients with exponential integral, we employ $\omega_t = \frac{\alpha_t}{\sigma_t}$  and transform \cref{eq:exp_int_x} { into \cref{eq:exp_int_x_omega2}} with a similar interpolation format as \cref{eq:fm_interpolate_cofun}, where $t(\omega)$ maps $\omega$ to timestep.
\begin{align}
    {\bs x}_t &\approx \frac{\sigma_t} {\sigma_s}\bar{\bs x}_s + \sigma_t ({\omega_t} - {\omega_s}) \sum_{k=1}^i c_i^k {\bs x}_\theta(\bar{\bs x}_k, t_k) \label{eq:exp_int_x_omega2} 
\end{align}

%% file: sections/experiment.tex
\section{Experiment}
We demonstrate the efficiency of our differentiable solver search by conducting experiments on publicly available diffusion models. Specifically, we utilize DiT-XL/2~\cite{dit} trained with DDPM scheduling and rectified-flow models SiT-XL/2~\cite{sit} and FlowDCN-XL/2~\cite{flowdcn}. Our default training setting employs the Lion optimizer~\cite{lion} with a constant learning rate of 0.01 and no weight decay. We sample 50,000 images for the entire search process. Notably, searching with 50,000 samples using FlowDCN-B/2 requires approximately 30 minutes on 8 × H20 computation cards. During the search, we deliberately avoid using CFG to construct reference and source trajectories, thereby preventing misalignment.

\subsection{Rectified Flow Models}
We search solver with FlowDCN-B/2, FlowDCN-S/2 and SiT-XL/2. We compare the search solver's performance with the second-order and fourth-order Adam multistep method on SiT-XL/2, FlowDCN-XL/2 trained on ImageNet $256\times256$ and FlowDCN-XL/2 trained on ImageNet $512\times512$. 

\input{tables/ablation_imagenet256}
\input{tables/fm_imagenet256}

\textbf{Search Model.} We tried different search models among different size and architecture. We report the FID performance of SiT-XL/2 in \cref{fig:model_fid}. Surprisingly, we find that the FID performance of SiT-XL/2 equipped with the solver searched using FlowDCN-B/2 outperforms the solver searched on SiT-XL/2 itself. The reconstruction error(in Appendix) between the sampled result produced by Euler-250 steps is as expected. These findings suggest that there exists a minor discrepancy between FID and the pursuit of minimal error in the current solver design.

\textbf{Step of Reference Trajectory.} We provide reference trajectory $\{ \tilde{\bs x} \}_{l=0}^L$ of different sampling step $L$ for differentiable solver search.  We take FlowDCN-B/2 as the search model and report the FID measured on SiT-XL/2 in \cref{fig:steps_fid}. As the sampling step of reference trajectory increases, the FID of SiT-XL/2 further improves and becomes better. However, the performance improvement is not significant when the number of steps is 5 or 6, which suggests that there is a limit to the improvement achievable with an extremely small number of steps.

\textbf{ImageNet $256\times256$.} We validate the searched solver on SiT-XL/2 and FlowDCN-XL/2. We arm the pre-trained model with CFG of 1.375. As shown in \cref{fig:sit256}, our searched solver improves FID performance significantly and achieves 2.40 FID under 10 steps. As shown in \cref{fig:flowdcn256}, our searched solver achieves 2.35 FID under 10 steps, beating traditional solvers by large margins. 

We also provide the comparison with recent solver-based distillation \cite{flowturbo} to demonstrate the efficiency of our searched solver in \cref{tab:distillation}. Our searched solver achieves better metric performance under similar NFE with much fewer parameters. 

\input{tables/fm_compare_to_distill}

\textbf{ImageNet $512\times512$.} Since \cite{sit} has not released SiT-XL/2 trained on $512\times512$ resolution, we directly report the performance collected from FlowDCN-XL/2. We arm FlowDCN-XL/2 with CFG of 1.375 and four channels.  Our searched solver achieves 2.77 FID under 10 steps, beating traditional solver by a large margin, even slightly outperforming the Euler solver with 50 steps(2.81FID).

\textbf{Text to Image.} Shown in \cref{fig:vis_flux_cfg2}, we apply our solver search on FlowDCN-B/2 and SiT-XL/2 to the most advanced Rectified-Flow model Flux.1-dev and SD3~\cite{sd3}.  We find Flux.1-Dev would produce grid points in generation. To alleviate the grid pattern, we decouple the velocity field into mean and direction, only apply our solver to the direction, and replace the mean with an exponential decayed mean. The details can be found in the appendix.

We also provide result of distillation on SD1.5 and SDXL with solver search in \cref{app:t2i_distillation_with_solver}.

\subsection{DDPM/VP Models}
\begin{figure*}
    \centering
    \includegraphics[width=0.99\linewidth]{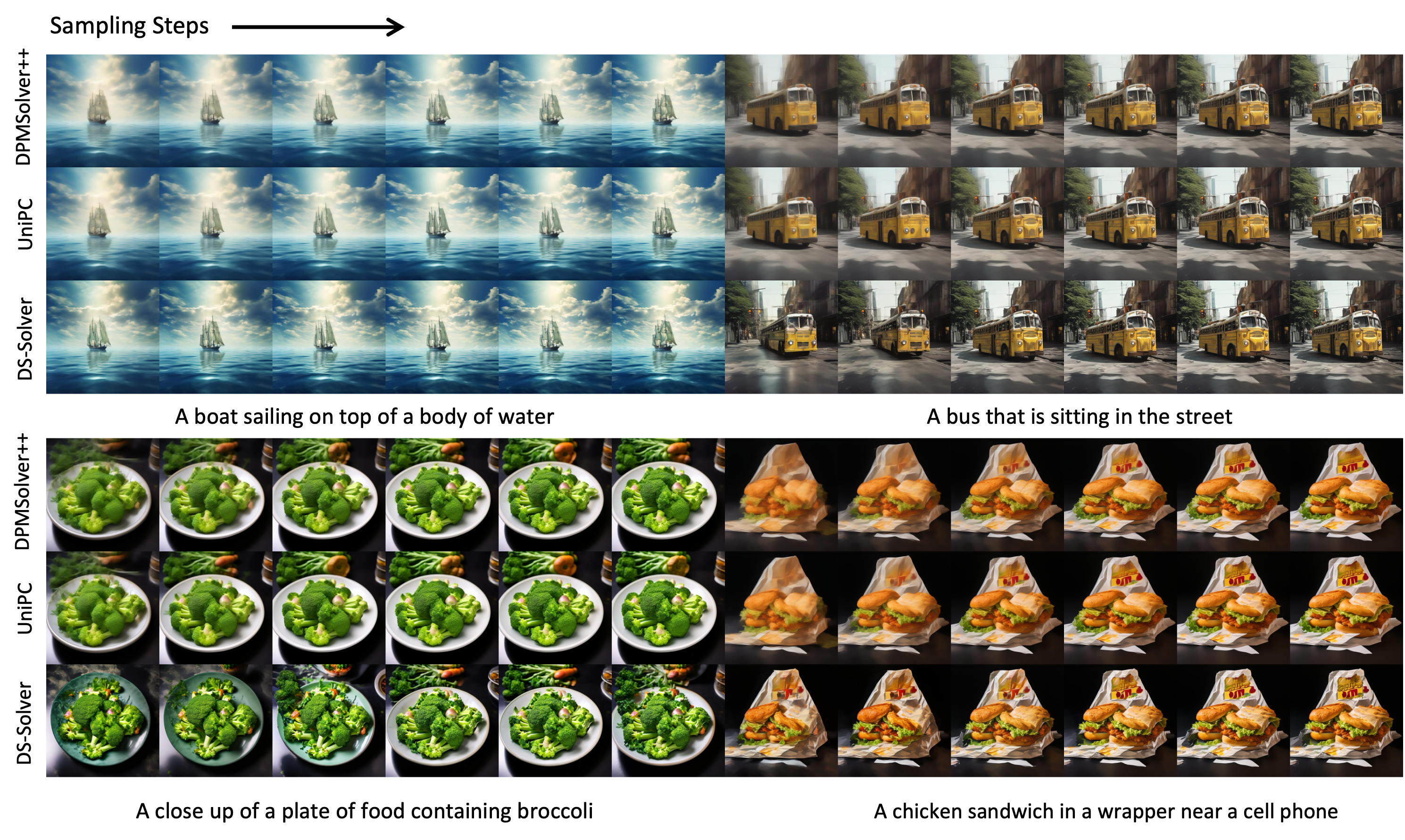}
    \caption{The images generated from PixArt-$\Sigma$ with CFG=2.0 equipped with Our DS-Solver ( searched on DiT-XL/2-R256 ).{\small In comparison to DPM-Solver++ and UniPC, our results consistently exhibit greater clarity and possess more details. Our solver achieves better quality from 5 to 10 steps(NFE).}}
    \label{fig:vis_pixart_512}
    \vspace{-1em}
\end{figure*}
\input{tables/dpm_imagenet256}
\input{tables/dpm_imagenet512}

We choose the open-source DiT-XL/2\cite{dit} model trained on ImageNet $256\times256$ as the search model to carry out the experiments. We compare the performance of the searched solver with DPM-Solver++ and UniPC on ImageNet $256\times256$ and ImageNet $512\times512$. 

\textbf{ImageNet $256\times256$.} Following \cite{dit} and \cite{top},  We arm pre-trained DiT-XL/2 with CFG of 1.5 and apply CFG only on the first three channels. As shown in \cref{tab: dpms and unipc_imagenet256}, our searched solver improves FID performance significantly and achieves 2.33 FID under 10 steps.

\textbf{ImageNet $512\times512$.} We directly apply the solver searched on $256\times256$ resolution to ImageNet $512\times512$. The result is also great to some extent, DiT-XL/2($512\times512$) achieves 3.64 FID under 10 steps, outperforming DPM-Solver++ and UniPC with a large gap. 

\textbf{Text to Image.} \label{sec:ddpm_t2i_exp} As we search solver with DiT and its corresponding noise scheduler, so it is infeasible to apply our solver to other DDPM models with different $\beta_\text{min}$ and $\beta_\text{max}$. Fortunately, we find \cite{pixart_sigma} and \cite{pixart} also employ the same $\beta_\text{min}$ and $\beta_\text{max}$ as DiT. So we can provide the visualization results of our searched solver on PixArt-$\Sigma$ and PixArt-$\alpha$. Our visualization result is produced with CFG of 2. 
\input{tables/geneval}
We take PixArt-alpha as the text-to-image model. We follow the evaluation pipeline of ADM and take COCO17-Val as the reference batch. We generate 5k images using DPM-Solver++, UniPC and our solver searched on DiT-XL/2-R256. Also, we provided the performance results on GenEval Benchmark~\cite{geneval} in \cref{tab:geneval_pixart}.

\input{tables/ddpm_t2i}

\subsection{Visualization Of Solver Parameters}

\textbf{Searched Coefficients} are visualized in \cref{fig:vis_solver}. The absolute value of searched coefficients corresponding to DDPM/VP shares a different pattern, coefficients in DDPM/VP are more concentrated on the diagonal while rectified-flow demonstrates a more flattened distribution. This indicates there exists a more curved sampling path in DDPM/VP compared to rectified-flow.

\textbf{Searched Timesteps} are visualized in \cref{fig:vis_solver}. Compared to DDPM/VP, rectified-flow models more focus on the more noisy region, exhibiting small time deltas at the beginning. We fit the searched timestep of different NFE with polynomials and provide the respacing curves in \cref{eq:respace_fm} and \cref{eq:respace_ddpm}. $t \in [0, 1]$, and $t=0$ indicates the most noisy timestep. 
\begin{align}
    & \text{ReFlow :}  -1.96t^4+3.51t^3-0.97t^2+0.43t   \label{eq:respace_fm} \\
    & \text{DDPM/VP :} -2.73 t^4 + 6.30 t^3 -4.744 t^2 + 2.17 t \label{eq:respace_ddpm}
    \vspace{-1em}
\end{align}

%% file: tables/ablation_imagenet256.tex
\begin{figure}
    \centering
   \subfloat[FID of Search Model\label{fig:model_fid}]{
    \includegraphics[width=0.48\linewidth]{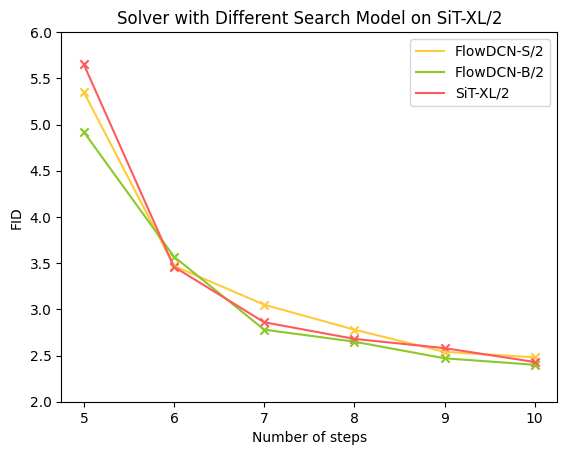}
    } 
   \subfloat[FID of RefTraj Steps\label{fig:steps_fid}]{
    \includegraphics[width=0.48\linewidth]{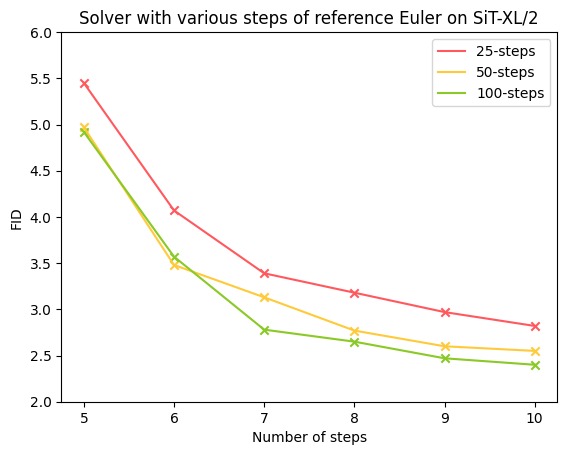}
    } 
    \caption{\textbf{Ablations studies of Differentiable Solver Search.} {\small We evaluate the searched solver on SiT-XL/2, and report the FID performance curve of searched solvers.}}
    \vspace{-2em}
\end{figure}

%% file: tables/fm_imagenet256.tex
\begin{figure*}
    \centering
   \subfloat[SiT-XL/2-R256 \label{fig:sit256}]{
        \includegraphics[width=0.3\linewidth]{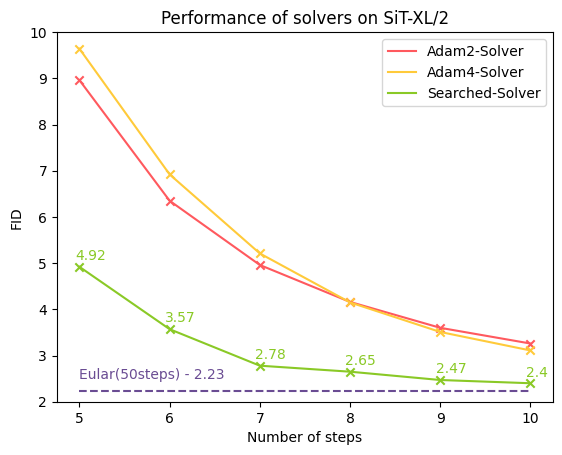}
    } 
   \subfloat[FlowDCN-XL/2-R256 \label{fig:flowdcn256}]{
        \includegraphics[width=0.3\linewidth]{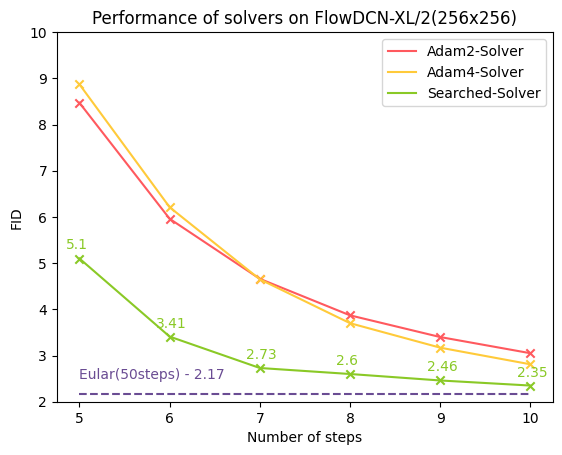}
    } 
    \subfloat[FlowDCN-XL/2-R512 \label{fig:flowdcn512}]{ 
         \includegraphics[width=0.3\linewidth]{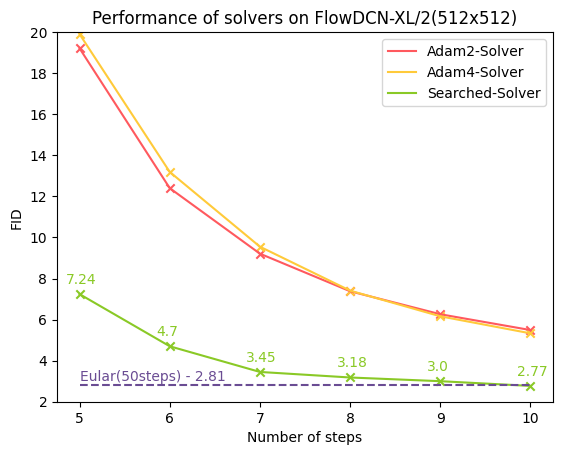}
    }
    
    \caption{\textbf{The same searched solver on different Rectified-Flow Models.} {\small R256 and R512 indicate the generation resolution of given model. We search solver with FlowDCN-B/2 on ImageNet-$256\times256$ and evaluate it with SiT-XL/2 and FlowDCN-XL/2 on different resolution datasets. Our searched solver outperforms traditional solvers by a significant margin. ~{More metrics(sFID, IS, Precision, Recall) are places at Appendix}}}
    \vspace{-1em}
\end{figure*}

%% file: tables/fm_compare_to_distill.tex
\begin{table}[ht]
\centering
\begin{tabular}{l|l|l|l|l|l}
\toprule
  SiT-XL-R256 & NFE-CFG & Params  & FID & IS   \\ 
\midrule
 Heun & 16x2 & 0 & 3.68 & /  \\   
 Heun & 22x2 & 0 & 2.79 & / \\   
 Heun & 30x2 & 0 & 2.42 & /  \\  
 Adam2 & 15x2 & / & 2.49 & 236 \\
 Adam2 & 16x2 & 0 & 2.42 & 237  \\ 
 Adam4 & 15x2 & / & 2.33 & 242 \\  
 Adam4 & 16x2 & 0 & 2.27 & 243 \\
 \midrule
 FlowTurbo & (7+3)x2 & $2.9\times10^7$ & 3.93 &  224  \\  
 FlowTurbo & (8+2)x2 & $2.9\times10^7$ & 3.63 & /  \\   
 FlowTurbo & (12+2)x2 & $2.9\times10^7$ & 2.69 & / \\  
 FlowTurbo & (17+3)x2 & $2.9\times10^7$ & 2.22 & 248 \\  \midrule
 ours & 6x2 & 21 & 3.57 & 214  \\ 
 ours & 7x2 & 28 & 2.78 & 229  \\ 
 ours & 8x2 & 36 & 2.65 & 234  \\ 
 ours & 10x2 & 55 & 2.40 & 238 \\ 
 ours & 15x2 & 55 & 2.24 & 244 \\ 
 \bottomrule
\end{tabular}

\caption{\textbf{Comparsion with Distillation methods.} {\small Our searched solver achieves much better result under the same NFE with much fewer parameters.}}
\label{tab:distillation}
\end{table}

%% file: tables/dpm_imagenet256.tex
\begin{table*}[hbt!]
\centering
\scalebox{0.85}{
\begin{tabular}{lcccccc}
\toprule
Methods \textbackslash NFEs & 5 & 6 & 7 & 8 & 9 & 10\\
\midrule
DPM-Solver++ with uniform-$\lambda$~\cite{dpmsolver++} & 38.04 & 20.96  & 14.69 & 11.09 &  8.32 & 6.47  \\
\midrule
DPM-Solver++ with uniform-$t$~\cite{dpmsolver++} & 31.32 & 14.36 & 7.62 & 4.93 &  3.77 & 3.23  \\
\midrule
DPM-Solver++ with uniform-$\lambda$-opt ~\cite{top} & 12.53 & 5.44 & 3.58 & 7.54 &  5.97 & 4.12  \\
\midrule
DPM-Solver++ with uniform-$t$-opt ~\cite{top} & 12.53 & 5.44 & 3.89 & 3.81 & 3.13 & 2.79 \\
\midrule
UniPC with uniform-$\lambda$~\cite{unipc} & 41.89 & 30.51 & 19.72 &12.94 &  8.49 & 6.13\\
\midrule
UniPC with uniform-$t$~\cite{unipc} & 23.48 & 10.31  & 5.73 & 4.06 &  3.39 & 3.04 \\
\midrule
UniPC with uniform-$\lambda$-opt~\cite{top} & 8.66 & 4.46 & 3.57 & 3.72 &  3.40 & 3.01 \\
\midrule
UniPC with uniform-$t$-opt ~\cite{top} & 8.66 & 4.46 & 3.74 & 3.29 & 3.01 & 2.74  \\
\midrule
\textbf{Searched-Solver} & \bf7.40 & \bf3.94 & \bf2.79 & \bf2.51 & \bf2.37 & \bf2.33 \\
\bottomrule
\end{tabular}
}
\caption{\textbf{FID ($\downarrow$) of different NFEs on DiT-XL/2-R256} . {\small \textit{-opt} indicates online optimization of the timesteps scheduler.}}
\vspace{-1em}
\label{tab: dpms and unipc_imagenet256}
\end{table*}

%% file: tables/dpm_imagenet512.tex
\begin{table*}[hbt!]
\centering
\scalebox{0.85}{
\begin{tabular}{lcccccc}
\toprule
Methods \textbackslash NFEs & 5 & 6 & 7 & 8 & 9 & 10\\
\midrule
UniPC with uniform-$\lambda$~\cite{unipc} & 41.14 & 19.81 & 13.01 & 9.83 &  8.31 & 7.01 \\
\midrule
UniPC with uniform-$t$~\cite{unipc} & 20.28 & 10.47  & 6.57 & 5.13 &  4.46 & 4.14\\
\midrule
UniPC with uniform-$\lambda$-opt~\cite{top} & 11.40 & \bf5.95 & 4.82 & 4.68 &  6.93 & 6.01 \\
\midrule
UniPC with uniform-$t$-opt~\cite{top} & 11.40 & \bf5.95 & 4.64 & 4.36 &  4.05 & 3.81\\
\midrule
\textbf{Searched-solver}{(searched on DiT-XL/2-R256)} & \bf10.28 & 6.02 & \bf4.31 & \bf3.74 & \bf3.54 & \bf3.64 \\
\bottomrule
\end{tabular}
}
\caption{\textbf{FID ($\downarrow$) of different NFEs on DiT-XL/2-R512.} {\small \textit{-opt} indicates online optimization of the timesteps scheduler.}}
\vspace{-1em}
\label{tab:unipc_imagenet512}
\end{table*}

%% file: tables/geneval.tex
\begin{table}
    \centering
    \setlength{\tabcolsep}{2pt}
    \label{tab:geneval_pixart}
    \begin{tabular}{lcccccc}
        \toprule
        \multirow{2}{*}{Solver} & \multirow{2}{*}{Steps} & \multirow{2}{*}{CFG} & \multicolumn{3}{c}{GenEval Metrics} & \multirow{2}{*}{Overall} \\
        &  &  & Color.Attr & Two.Obj & Pos &  \\
        \midrule
        \multirow{2}{*}{DPM++} & 5 & 2.0 & 6.50 & 33.08 & 4.75 & 0.40519 \\
        & 8 & 2.0 & 5.25 & 39.65  & 5.75 & 0.43074 \\
        \midrule
        \multirow{2}{*}{UniPC} & 5 & 2.0 & 6.50 & 34.85 & 5.25 & 0.41387 \\
        & 8 & 2.0 & 6.72 & 40.66 & 6.00 & 0.44134 \\
        \midrule
        \multirow{2}{*}{Ours} & 5 & 2.0 & 5.25 & 37.37  & 4.75 & 0.41933 \\
        & 8 & 2.0 & 7.25 & 42.68 & 7.50 & 0.45064 \\
        \bottomrule
    \end{tabular}
    \caption{\textbf{Results on GenEval Benchmark for PixArt at 512 Resolution.}{\small Our searched solver achieves better performance compared with UniPC/DPM++ on PixArt-$512\times512$.}}
\end{table}

%% file: tables/ddpm_t2i.tex
\begin{table}
\centering
\begin{tabular}{l|l|l|l|l|l|l}
\toprule
      & Steps  & FID & sFID & IS & PR & Recall  \\ 
\midrule
DPM++ & 5 & 60.0 & 209 & 25.59 & 0.36 & 0.20 \\
DPM++ & 8 & 38.4 & 116.9 & 33.0 & 0.50 & 0.36 \\ 
DPM++ & 10 & 35.6 & 114.7 & 33.7 & 0.53 & 0.37 \\ 
\midrule
UniPC & 5 & 57.9 & 206.4 & 25.88 & 0.38 & 0.20 \\ 
UniPC & 8 & 37.6 & 115.3 & 33.3 & 0.51 & 0.36 \\ 
UniPC & 10 & 35.3 & 113.3 & 33.6 & 0.54 & 0.36\\ 
\midrule
Ours & 5 & 46.4 & 204 & 28.0 & 0.46 & 0.23\\ 
Ours & 8 & 33.6 & 115.2 & 32.6 & 0.54 & 0.39\\ 
Ours & 10 & 33.4 & 114.7 & 32.5 & 0.55 & 0.39 \\ 
\bottomrule
\end{tabular}
\caption{Metrics of different NFEs on PixArt-$\alpha$ (Our Solver are searched on ImageNet 256x256).} 
\vspace{-2em}
\end{table}

%% file: sections/conclusion.tex
\section{Conclusion}
\vspace{-0.3em}
We have found a compact solver search space and proposed a novel differentiable solver search algorithm to identify the optimal solver. Our searched solver outperforms traditional solvers by a significant margin. Equipped with the searched solver, DDPM/VP and Rectified Flow models significantly improve under limited sampling steps. However, our proposed solver still has several limitations which we plan to address in future work.

%% file: sections/appendix.tex
\section{More Metrics of Searched Solver} 
We adhere to the evaluation guidelines provided by ADM and DM-nonuniform, reporting only the FID as the standard metric in \cref{fig:sit256}. To clarify, we do not report selective results on rectified flow models; we present sFID, IS, PR, and Recall metrics for SiT-XL(R256), FlowDCN-XL/2(R256), and FlowDCN-B/2(R256). Our solver searched on FlowDCN-B/2, consistently outperforms the handcrafted solvers across FID, sFID, IS, and Recall metrics.
\begin{figure}
    \centering
    \subfloat[SiT-XL/2-R256] {
        \includegraphics[width=0.3\linewidth]{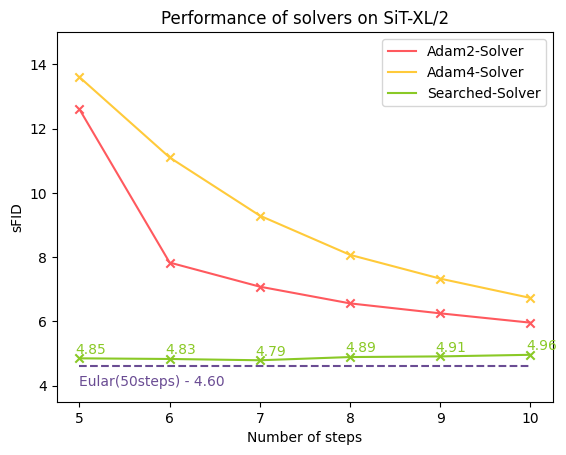}
    } 
   \subfloat[FlowDCN-XL/2-R256] {
        \includegraphics[width=0.3\linewidth]{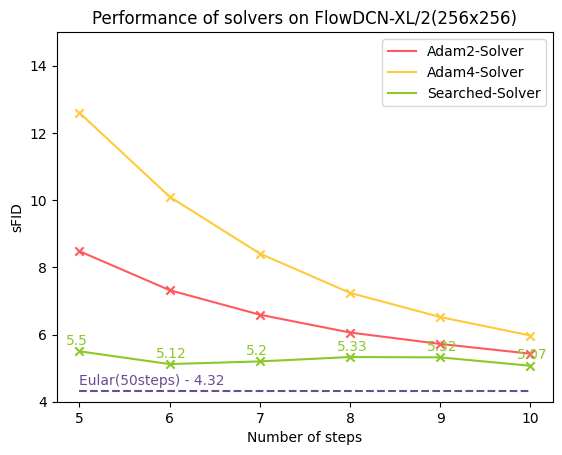}
    } 
    \subfloat[FlowDCN-XL/2-R512]{ 
         \includegraphics[width=0.3\linewidth]{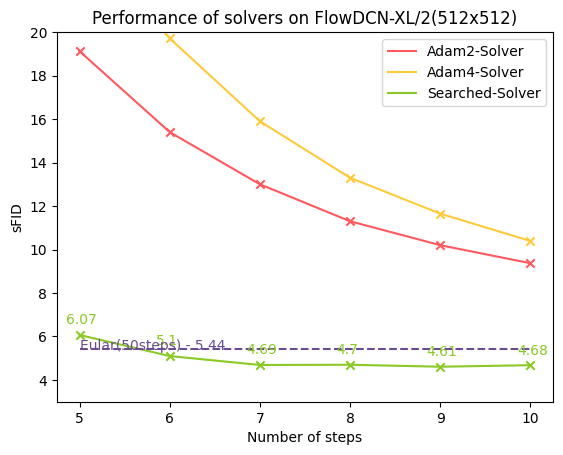}
    } 
    
    \subfloat[SiT-XL/2-R256] {
        \includegraphics[width=0.3\linewidth]{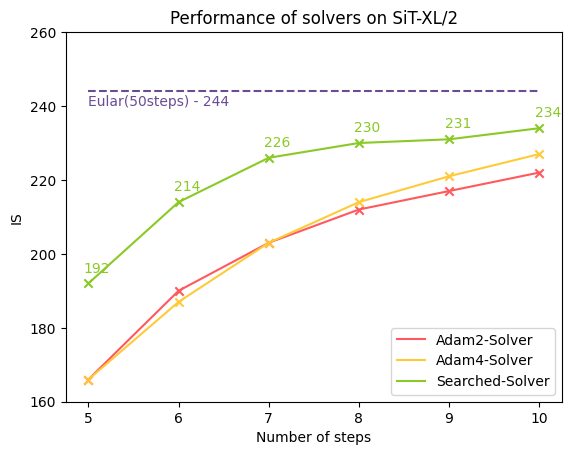}
    } 
   \subfloat[FlowDCN-XL/2-R256] {
        \includegraphics[width=0.3\linewidth]{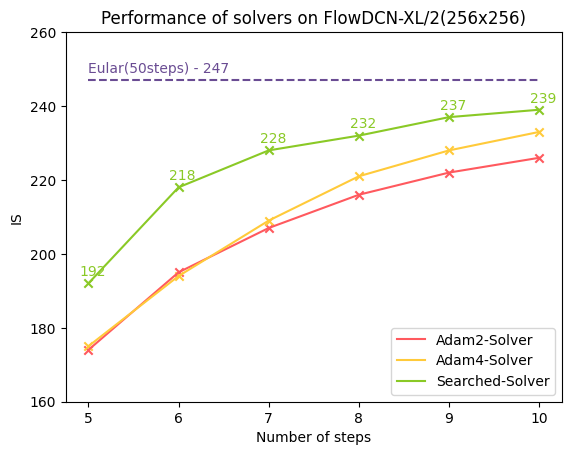}
    } 
    \subfloat[FlowDCN-XL/2-R512]{ 
         \includegraphics[width=0.3\linewidth]{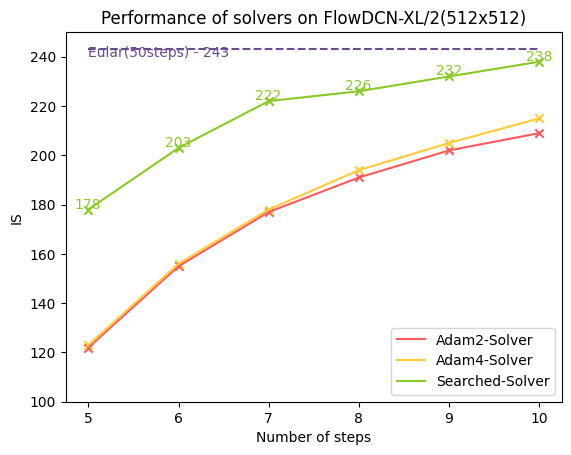}
    } 
    
    \subfloat[SiT-XL/2-R256] {
        \includegraphics[width=0.3\linewidth]{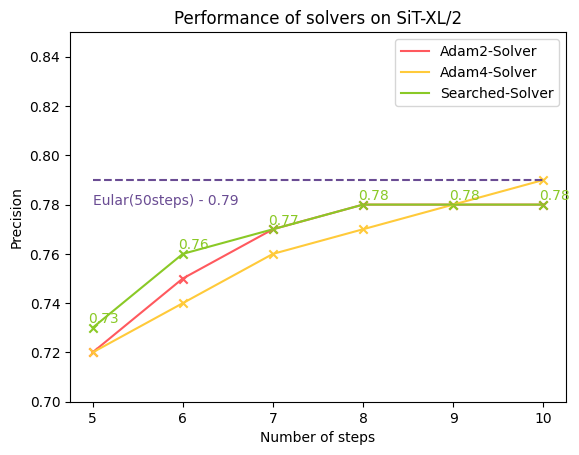}
    } 
   \subfloat[FlowDCN-XL/2-R256] {
        \includegraphics[width=0.3\linewidth]{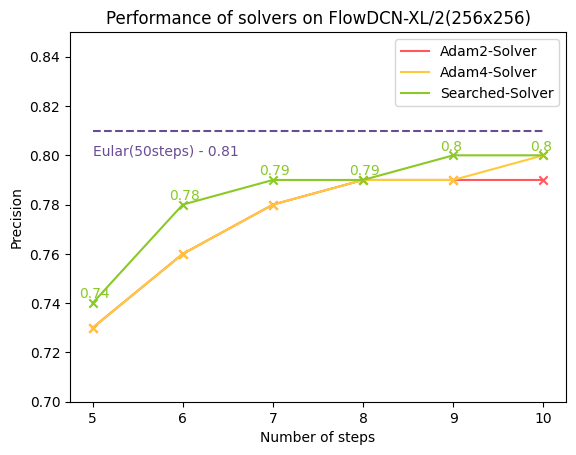}
    } 
    \subfloat[FlowDCN-XL/2-R512]{ 
         \includegraphics[width=0.3\linewidth]{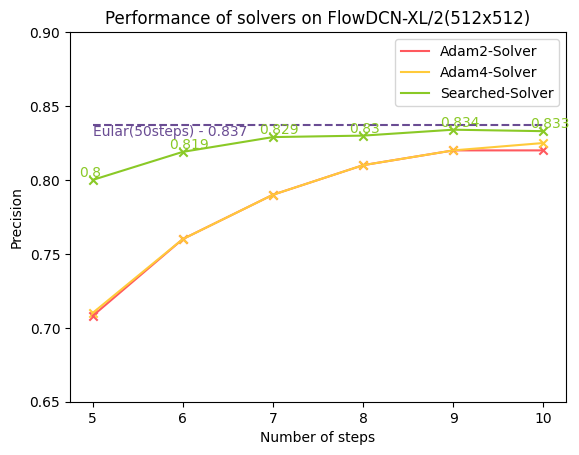}
    } 
    
    \subfloat[SiT-XL/2-R256] {
        \includegraphics[width=0.3\linewidth]{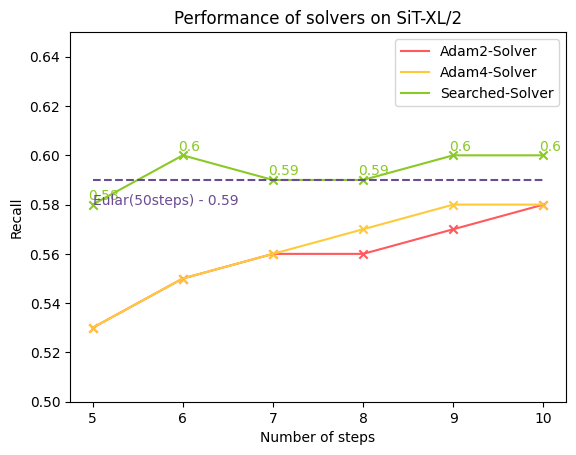}
    } 
   \subfloat[FlowDCN-XL/2-R256] {
        \includegraphics[width=0.3\linewidth]{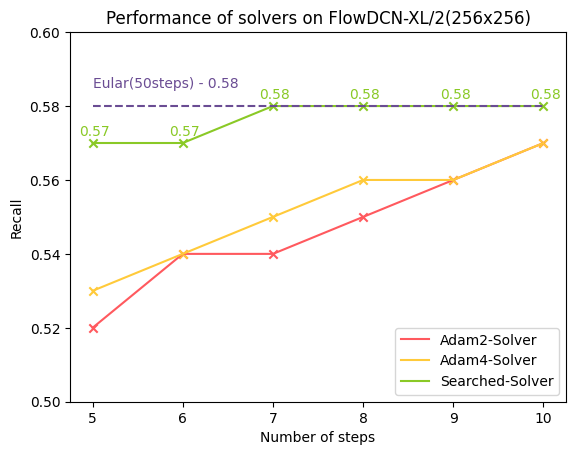}
    } 
    \subfloat[FlowDCN-XL/2-R512]{ 
         \includegraphics[width=0.3\linewidth]{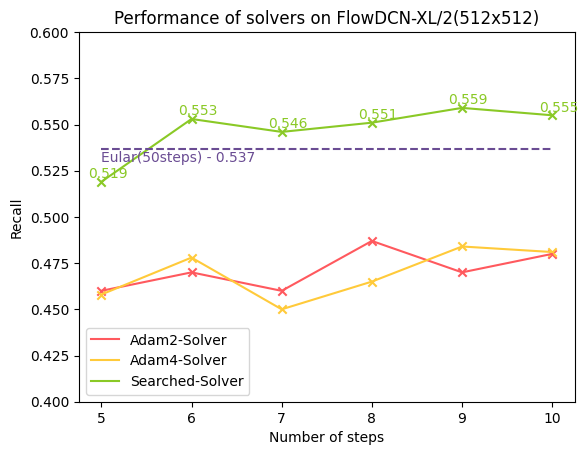}
    } 
   
    \caption{\textbf{The same searched solver on different Rectified-Flow Models.} \label{fig:all_sota} {\small R256 and R512 indicate the generation resolution of given model. We search solver with FlowDCN-B/2 on ImageNet-$256\times256$ and evaluate it with SiT-XL/2 and FlowDCN-XL/2 on different resolution datasets. Our searched solver outperforms traditional solvers by a significant margin.}}
    \vspace{-2em}
\end{figure}

\section{Computational complexity compared to other methods.}

\textbf{For sampling.}  
When performing sampling over $n$ time steps, our solver caches all pre-sampled predictions, resulting in a memory complexity of  $\mathcal{O}(n)$. The model function evaluation also has a complexity of $\mathcal{O}(n)$ ($\mathcal{O}(2 \times n)$ for CFG enabled). It is important to note that the memory required for caching predictions is negligible compared to that used by model weights and activations. Besides classic methods, we have also included a comparison with the latest Flowturbo published on NeurIPS24.

\begin{table}[h]
\centering
\begin{tabular}{|l|l|l|l|l|l|l|}
\hline
      & Steps  & NFE & NFE-CFG & Cache Pred & Order & search samples \\ \hline
Adam2 &  n & n & 2n & 2 & 2& / \\ \hline
Adam4 & n & n & 2n & 4 & 4  & / \\ \hline
heun & n & 2n & 4n & 2 & 2 & /\\ \hline
DPM-Solver++ &  n & n & 2n & 2 & 2& / \\ \hline
UniPC &  n & n & 2n & 3 & 3 & / \\ \hline
FlowTurbo &  n & $>$n & $>$2n & 2 & 2 & 540000(Real) \\ \hline
our & n & n & 2n & n & n& 50000(Generated) \\ \hline
\end{tabular}
\end{table}

\textbf{For searching.}  Solver-based algorithms, limited by their searchable parameter sizes, demonstrate significantly lower performance in few-step settings compared to distillation-based algorithms(5/6steps), making direct comparisons inappropriate. Consequently, we selected algorithms that are both acceleratable on ImageNet and comparable in performance, including popular methods such as DPM-Solver++, UniPC, and classic Adams-like linear multi-step methods. Since our experiments primarily utilize SiT, DiT, and FlowDCN that trained on the ImageNet dataset. We also provide fair comparisons by incorporating the latest acceleration method, FlowTurbo. Additionally, we have included results from the heun method as reported in FlowTurbo.

\section{Ablation on Search Samples}

We ablate the number of search samples on the 10-step and 8-step solver settings. \textit{Samples} means the total training samples the searched solver has seen.   \textit{Unique Samples} means the total distinct samples the searched solver has seen.  Our searched solver converges fast and gets saturated near 30000 samples.
\begin{table}[h]
\centering
\begin{tabular}{|l|l|l|l|l|l|l|}
\hline
  iters(10-step-solver)  & samples & unique samples  & FID & IS & PR & Recall  \\ \hline
  313 & 10000 & 10000 & 2.54 & 239 & 0.79 & 0.59\\\hline
  626 & 20000 & 10000 & 2.38 & 239 & 0.79 & 0.60\\\hline
  939 & 30000 & 10000 & 2.49 & 240 & 0.79 & 0.59\\\hline
  1252 & 40000 & 10000 & 2.29 & 239 & 0.80 & 0.60\\\hline
  1565 & 50000 & 10000 & 2.41 & 240 & 0.80 & 0.59\\\hline
  626 & 20000 & 20000 & 2.47 & 237 & 0.78 & 0.60\\\hline
  939 & 30000 & 30000 & 2.40 & 238 & 0.79 & 0.60\\\hline
  1252 & 40000 & 40000 & 2.48 & 237 & 0.80 & 0.59\\\hline
  1565 & 50000 & 50000 & 2.41 & 239 & 0.80 & 0.59\\\hline
\end{tabular}
\end{table}

\begin{table}[h]
\centering
\begin{tabular}{|l|l|l|l|l|l|l|}
\hline
  iters(8-step-solver)  & samples & unique samples  & FID & IS & PR & Recall  \\ \hline
  313 & 10000 & 10000 & 2.99 & 228 & 0.78 & 0.59\\\hline
  626 & 20000 & 10000 & 2.78 & 229 & 0.79 & 0.60\\\hline
  939 & 30000 & 10000 & 2.72 & 235 & 0.79 & 0.60\\\hline
  1252 & 40000 & 10000 & 2.67 & 228 & 0.79 & 0.60\\\hline
  1565 & 50000 & 10000 & 2.69 & 235 & 0.79 & 0.59\\\hline
  626 & 20000 & 20000 & 2.70 & 231 & 0.79 & 0.59\\\hline
  939 & 30000 & 30000 & 2.82 & 232 & 0.79 & 0.59\\\hline
  1252 & 40000 & 40000 & 2.79 & 231 & 0.79 & 0.60\\\hline
  1565 & 50000 & 50000 & 2.65 & 234 & 0.79 & 0.60\\\hline
\end{tabular}
\end{table}

\section{Solver Across different variance schedules}

Since our solvers are searched on a specific noise scheduler and its corresponding pre-trained models, applying the searched coefficients and timesteps to other noise schedulers yields meaningless results. We have tried applied searched solver on SiT(Rectified flow) and DiT(DDPM with $\beta_{min}=0.1, \beta_{max}=20$) to SD1.5(DDPM with $\beta_{min}=0.085, \beta_{max}=12$), but the results were inconclusive. Notably, despite sharing the DDPM name, DiT and SD1.5 employ distinct $\beta_{min}, \beta_{max}$ values, thereby featuring different noise schedulers. A more in-depth discussion of these experiments can be found in Section(Extend to DDPM/VP).

\section{Solver for different variance schedules}
As every DDPM has a corresponding continuous VP scheduler, so we can transform the discreet DDPM into continuous VP, thus we successfully searched better solver compared to DPM-Solvers. The details can be found in \cref{sec:ddpm_vp_extend}. To put it simply, under the empowerment of our high-order solver, the performance of DDPM and FM does not differ significantly (8, 9, 10 steps), which contradicts the common belief that FM is stronger at limited sampling steps.

\section{Text to image Distillation Experiments}
\label{app:t2i_distillation_with_solver}
We unify distillation and solver search to obtain high-quality multi-step generative models. We adopt adversarial training and trajectory supervision. We will open source the training code of unified training techniques.

\begin{table*}[h]
\centering
\small
\caption{\small Performance comparison on validation set of COCO-2017.}
\label{tab:coco-2017}
\resizebox{.85\textwidth}{!}{
\begin{tabular}{@{}llccccc@{}}
\toprule
\textbf{Method} & \textbf{Res.} & \textbf{Time ($\downarrow$)} & \textbf{\# Steps} & \textbf{\# Param.} & \textbf{FID ($\downarrow$)}\\ \midrule
SDv1-5+DPMSolver~(Upper-Bound)~\citep{dpmsolver} & 512& 0.88s & 25 & 0.9B & 20.1 \\\midrule
Rectified Flow & 512& 0.88s & 25 & 0.9B & 21.65\\ 
Rectified Diffusion & 512& 0.88s & 25 & 0.9B & 21.28\\ 
Rectified Flow& 512& 0.21s & 4 & 0.9B & 103.48 \\ 
PeRFlow& 512& 0.21s & 4 & 0.9B & 22.97 \\ 
Rectified Diffusion& 512& 0.21s & 4 & 0.9B & 20.64\\ 
Ours(Distillation+solver search)& 512 & 0.21s & 4 & 0.9B & 18.99 \\
\midrule 
PeRFlow-SDXL& 1024& 0.71s & 4 & 3B & 27.06 \\ 
Rectified Diffusion-SDXL& 1024& 0.71s & 4 & 3B & 25.81 \\
Ours(LORA+Distillation+solver search)& 1024 & 0.71s & 4 & 3B & 21.3 \\
\bottomrule
\end{tabular}}
\end{table*}

\begin{table*}[htp]
\centering
\small
\caption{\small Performance comparison on COCO-2014.}
\label{tab:comparison}
\resizebox{.85\textwidth}{!}{
\begin{tabular}{@{}llcccc@{}}
\toprule
\textbf{Method} & \textbf{Res.} & \textbf{Time ($\downarrow$)} & \textbf{\# Steps} & \textbf{\# Param.} & \textbf{FID ($\downarrow$)} \\ \midrule
\midrule
    \multicolumn{6}{c}{Stable Diffusion XL~(3B) and its accelerated or distilled versions}\\ 
    SDXL-Turbo & 512& 0.34s & 4 & 3B & 23.19\\ 
    SDXL-Lightning & 1024& 0.71s & 4 & 3B & 24.56\\ 
    DMDv2 & 1024& 0.71s & 4 & 3B & 19.32\\ 
    LCM & 1024& 0.71s & 4 & 3B & 22.16\\ 
    Phased Consistency Model& 1024& 0.71s & 4 & 3B & 21.04\\ 
    PeRFlow-XL& 1024& 0.71s & 4 & 3B & 20.99 \\ 
    Rectified Diffusion-XL~(Phased)& 1024& 0.71s & 4 & 3B & 19.71\\ 
    Ours(LORA+Distillation+solver search)& 1024& 0.71s & 4 & 3B & 11.4\\ 
\bottomrule
\end{tabular}}
\end{table*}

\section{Limitations.}
We place the limitation at the appendix, in order to provide more discussion space and obtain more insights from reviews. We copy the original limitation content and add more.

\textbf{Misalignd Reconstrucion loss and Performance.} Our proposed methods are specifically designed to minimize integral error within a limited number of steps. However, ablation studies reveal a mismatch between FID performance and Reconstruction error. To address this issue, we plan to enhance our searched solver by incorporating distribution matching supervision, thereby better aligning sampling quality.

\textbf{Larger CFG Inference.} In the main paper, we demonstrate text-to-image visualization with a small CFG value. However, it is intuitive that utilizing a larger CFG would result in superior image quality. We attribute the inferior performance of large CFGs on our solver to the limitations of current naive solver structures and searching techniques. We hypothesize that incorporating predictor-corrector solver structures would enhance numerical stability and yield better images. Additionally, training with CFGs may also be beneficial.

\textbf{Resource Consumption} We can hard code the searched coefficients and timesteps into the program files. However, Compared to hand-crafted solvers, our solver still needs a searching process.

\section{Proof of pre-integral error expectation}
\label{app:proof_pre_error}
\begin{theorem}
Given sampling time interval $[t_i, t_{i+1}]$ and suppose $\mathcal{C}_j({\bs x}) = g_j({\bs x}) + b_i^j$, Adams-like linear multi-step methods will introduce an upper error bound of $(t_{i+1} - t_i)\mathbb{E}_{{\bs x}_i}||\sum_{j=0}^i {\bs v}_j g_j({\bs x}_i) || $.

Our solver search(replacing $\mathcal{C}_j({\bs x})$ with $\mathbb{E}_{{\bs x}_i}[\mathcal{C}_j({\bs x}_i)]$) owns an upper error bound of $(t_{i+1} - t_i)\mathbb{E}_{{\bs x}_i}||\sum_{j=0}^i {\bs v}_j [g_j({\bs x}_i) -  \mathbb{E}_{{\bs x}_i}g_j({\bs x}_i) || $
\end{theorem}

\begin{proof}
Suppose $\mathcal{C}_j({\bs x_i}) = g_j({\bs x_i}) + b^j_i$. Adams-like linear multi-step methods would not consider $x$-related interpolation. thus pre-integral coefficients of Adams-like linear multi-step methods will only reduce into $b$. 

We obtain the error expectation of the pre-integral of Adams-like linear multi-step methods: 
\begin{align}
    &\mathbb{E}_{{\bs x}_i}||\sum_{j=0}^i {\bs v}_j [\mathcal{C}_j({\bs x}_i)] (t_{i+1} - t_i) -  \sum_{j=0}^i {\bs v}_j b_i^j (t_{i+1} - t_i) || \\
    =&\mathbb{E}_{{\bs x}_i}||\sum_{j=0}^i {\bs v}_j (t_{i+1} - t_i) [\mathcal{C}_j({\bs x}_i) - b_i^j  || \\
    =&(t_{i+1} - t_i) \mathbb{E}_{{\bs x}_i}||\sum_{j=0}^i {\bs v}_j  g_j({\bs x}_i)  ||
\end{align}
We obtain the error expectation of the pre-integral of our solver search methods:
\begin{align}
    &\mathbb{E}_{{\bs x}_i}||\sum_{j=0}^i {\bs v}_j [\mathcal{C}_j({\bs x}_i)] (t_{i+1} - t_i) -  \sum_{j=0}^i {\bs v}_j \mathbb{E}_{{\bs x}_i}[\mathcal{C}_j({\bs x}_i)] (t_{i+1} - t_i) || \\
    =&\mathbb{E}_{{\bs x}_i}||\sum_{j=0}^i {\bs v}_j (t_{i+1} - t_i) [\mathcal{C}_j({\bs x}_i) - \mathbb{E}_{{\bs x}_i}\mathcal{C}_j({\bs x}_i) || \\
    =& (t_{i+1} - t_i)\mathbb{E}_{{\bs x}_i}||\sum_{j=0}^i {\bs v}_j [g_j({\bs x}_i) -  \mathbb{E}_{{\bs x}_i}g_j({\bs x}_i) || 
\end{align}
Next, define the optimization problem:
$$E = \mathbb{E}_{{\bs x}_i}||\sum_{j=0}^i {\bs v}_j [g_j({\bs x}_i) - a_j]||_2^2.$$
We suppose different $v_j$ are orthogonal and $||v_j||_2^2 = 1$. As we leave $c_j^i$ as the expectation of $\mathcal{C}_j({\bs x}_i)$, we will demonstrate this choice is optimal.
\begin{equation}
 \frac{\partial E}{\partial a_j} = -2\mathbb{E}_{{\bs x}_i}(||v_j||_2^2 (g_j(x_i) - a_j))    
\end{equation}
Let $\frac{\partial E}{\partial a_j} = 0$, we obtain: $a_j = \frac{\mathbb{E}_{{\bs x}_i}g_i(x_i)||v_j||_2^2}{\mathbb{E}_{{\bs x}_i}||v_j||_2^2} = \mathbb{E}_{{\bs x}_i}g_j({\bs x}_i) = \mathbb{E}_{{\bs x}_i}\mathcal{C}_j({\bs x}_i) - b_i^j$.

So our searched solver has a lower and optimal error expectation:
\begin{equation}
     (t_{i+1} - t_i)\mathbb{E}_{{\bs x}_i}||\sum_{j=0}^i {\bs v}_j [g_j({\bs x}_i) -  \mathbb{E}_{{\bs x}_i}g_j({\bs x}_i)]|| \leq (t_{i+1} - t_i) \mathbb{E}_{{\bs x}_i}||\sum_{j=0}^i {\bs v}_j  g_j({\bs x}_i) || 
\end{equation}

Recall \cref{ass:interpolation_bound}, the integral upper error bound of universal interpolation $\mathcal{P}$ will be: 
\begin{align}
   &||\int_{t_i}^{t_{i+1}} v({\bs x_t}, t) dt  -  \sum_{j=0}^i {\bs v}_j \int_{t_i}^{t_{i+1}} \mathcal{P}({\bs x}_t, t, {\bs x}_j, t_j) dt ||. \\
   =&||\int_{t_i}^{t_{i+1}} v({\bs x_t}, t) dt - \int_{t_i}^{t_{i+1}} \sum_{j=0}^i \mathcal{P}({\bs x}_t, t, {\bs x}_j, t_j){\bs v}_j dt|| .  \\
   =&||\int_{t_i}^{t_{i+1}} [v({\bs x_t}, t) - \sum_{j=0}^i \mathcal{P}({\bs x}_t, t, {\bs x}_j, t_j){\bs v}_j] dt|| . \\
   <&\int_{t_i}^{t_{i+1}} ||v({\bs x_t}, t) - \sum_{j=0}^i \mathcal{P}({\bs x}_t, t, {\bs x}_j, t_j){\bs v}_j|| dt . \\
   <& (t_{i+1} - t_{i})[\mathcal{O}(d{\bs x}^m) + \mathcal{O}({dt}^n)] \label{reb:p_error_bound} 
\end{align}

Combining \cref{reb:p_error_bound} and the error expectation of the pre-integral part, we will get the total error bound of the solver search.
\begin{align}
    &||\int_{t_i}^{t_{i+1}} v({\bs x_t}, t) dt  -  \sum_{j=0}^i {\bs v}_j \mathbb{E}_{{\bs x}_i}[\mathcal{C}_j({\bs x}_i)] (t_{i+1} - t_i) ||. \\
  =&||\int_{t_i}^{t_{i+1}} v({\bs x_t}, t) dt  - \sum_{j=0}^i {\bs v}_j \int_{t_i}^{t_{i+1}} \mathcal{P}({\bs x}_t, t, {\bs x}_j, t_j) dt + \\ &\sum_{j=0}^i {\bs v}_j \int_{t_i}^{t_{i+1}} \mathcal{P}({\bs x}_t, t, {\bs x}_j, t_j) dt -  \sum_{j=0}^i {\bs v}_j \mathbb{E}_{{\bs x}_i}[\mathcal{C}_j({\bs x}_i)] (t_{i+1} - t_i) ||. \\
  <&||\int_{t_i}^{t_{i+1}} v({\bs x_t}, t) dt  - \sum_{j=0}^i {\bs v}_j \int_{t_i}^{t_{i+1}} \mathcal{P}({\bs x}_t, t, {\bs x}_j, t_j) dt|| + \\ &||\sum_{j=0}^i {\bs v}_j \int_{t_i}^{t_{i+1}} \mathcal{P}({\bs x}_t, t, {\bs x}_j, t_j) dt -  \sum_{j=0}^i {\bs v}_j \mathbb{E}_{{\bs x}_i}[\mathcal{C}_j({\bs x}_i)] (t_{i+1} - t_i) ||. \\
  =&||\int_{t_i}^{t_{i+1}} v({\bs x_t}, t) dt  - \sum_{j=0}^i {\bs v}_j \int_{t_i}^{t_{i+1}} \mathcal{P}({\bs x}_t, t, {\bs x}_j, t_j) dt|| + \\ &||\sum_{j=0}^i {\bs v}_j [\mathcal{C}_j({\bs x}_i)] (t_{i+1} - t_i) -  \sum_{j=0}^i {\bs v}_j \mathbb{E}_{{\bs x}_i}[\mathcal{C}_j({\bs x}_i)] (t_{i+1} - t_i) ||. \\
  <& {(t_{i+1} - t_{i})}[\mathcal{O}(d{\bs x}^m) + \mathcal{O}({dt}^n)] + {(t_{i+1} - t_{i})} \mathbb{E}_{{\bs x}_i}||\sum_{j=0}^i {\bs v}_j  [g_j({\bs x}_i) - \mathbb{E}_{{\bs x}_i}g_j({\bs x}_i)] ||   \\
  <& {(t_{i+1} - t_{i})}([\mathcal{O}(d{\bs x}^m) + \mathcal{O}({dt}^n)] +  \mathbb{E}_{{\bs x}_i}||\sum_{j=0}^i {\bs v}_j  [g_j({\bs x}_i) - \mathbb{E}_{{\bs x}_i}g_j({\bs x}_i)]  ||)  
\end{align}

Since $((\mathcal{O}(d{\bs x}^m) + \mathcal{O}({dt}^n))$ is much smaller than $\mathbb{E}_{{\bs x}_i}||\sum_{j=0}^i {\bs v}_j  [g_j({\bs x}_i) - \mathbb{E}_{{\bs x}_i}g_j({\bs x}_i)]  ||$. We can omit the $((\mathcal{O}(d{\bs x}^m) + \mathcal{O}({dt}^n))$ term.

\end{proof}

\section{Proof of total upper error bound}
\label{app:proof_total_error}

\begin{theorem}
 Compared to Adams-like linear multi-step methods. Our Solver search has a small upper error bound. 
 
The total upper error bound of Adams-like linear multi-step methods is:
    $$
\sum_{i=0}^{N-1} (\frac{1}{N}) \sum_{j=0}^i \eta |b_i^j| + \mathbb{E}_{{\bs x}_i}||\sum_{j=0}^i {\bs v}_j  [g_j({\bs x}_i)] ||)
    $$
The total upper error bound of Our solver search is:
    $$
\sum_{i=0}^{N-1} (t_{i+1} - t_{i}) \sum_{j=0}^i \eta |\mathbb{E}_{{\bs x}_i}g_j({\bs x}_i) + b_i^j| + \mathbb{E}_{{\bs x}_i}||\sum_{j=0}^i {\bs v}_j g_j({\bs x}_i) - \mathbb{E}_{{\bs x}_i}g_j({\bs x}_i) ||) 
    $$ 
\end{theorem}
\begin{proof}
We donate the continuous integral result of the ideal velocity field $\hat {\bs v}$ as $\hat{\bs x}$, the solved integral result of the ideal velocity field $\hat {\bs v}$ as $\hat{\bs x}_N$, the continuous integral result of the pre-trained velocity model ${\bs v}_\theta$ as $\hat{\bs x}$, the solved integral result of the pre-trained velocity model ${\bs v}_\theta$ as ${\bs x}_N$. 

\begin{equation}
     {\bs x}_N = {\bs \epsilon} +  \sum_{i=0}^{N-1}\sum_{j=0}^i {\bs v}_j c_i^j (t_{i+1} - t_{i}) 
\end{equation}
The error caused by the non-ideal velocity estimation model can be formulated in the following equation. we can employ triangular inequalities to obtain the error-bound ${||{\bs x}_N - {\bs \hat{x}}_N||}$, which is related to solver coefficients and timestep choices. 
\begin{align*}
     { ||{\bs x}_N - {\hat{\bs x}}_N||} &= |\sum_{i=0}^{N-1}\sum_{j=0}^i ({\bs v}_j - \hat{{\bs v}}_j) c_i^j (t_{i+1} - t_{i})| \\
     &\leq \sum_{i=0}^{N-1}\sum_{j=0}^i |({\bs v}_j - \hat{{\bs v}}_j)c_i^j(t_{i+1} - t_{i})| \\
     &\leq \sum_{i=0}^{N-1}\sum_{j=0}^i |{\bs v}_j - \hat{{\bs v}}_j)| \times |c_i^j(t_{i+1} - t_{i})| \\ 
     &\leq \eta \sum_{i=0}^{N-1}\sum_{j=0}^i |c_i^j (t_{i+1} - t_{i})|
\end{align*}

The total error of our searched solver is:
\begin{align*}
     & {||{\bs x}_N - {\hat{\bs x}}||} \\
     =& ||{\bs x}_N -{\hat{\bs x}}_N + {\hat{\bs x}}_N - {\hat{\bs x}}|| \\
     \leq& { ||{\bs x}_N - {\hat{\bs x}}_N||} + {||{\hat{\bs x}}_N - {\hat{\bs x}}||} \\
     \leq& \eta \sum_{i=0}^{N-1}\sum_{j=0}^i |c_i^j (t_{i+1} - t_{i})| + \\
     & \sum_{i=0}^{N-1} {(t_{i+1} - t_{i})}(\mathcal{O}(d{\bs x}^m) + \mathcal{O}({dt}^n) +  \mathbb{E}_{{\bs x}_i}||\sum_{j=0}^i {\bs v}_j  [g_j({\bs x}_i) - \mathbb{E}_{{\bs x}_i}g_j({\bs x}_i)] ||)  \\
    \approx & \sum_{i=0}^{N-1}\eta \sum_{j=0}^i |c_i^j (t_{i+1} - t_{i})| + {(t_{i+1} - t_{i})}\mathbb{E}_{{\bs x}_i}||\sum_{j=0}^i {\bs v}_j  [g_j({\bs x}_i) - \mathbb{E}_{{\bs x}_i}g_j({\bs x}_i)]  ||)  \\
    = & \sum_{i=0}^{N-1} (t_{i+1} - t_{i}) \sum_{j=0}^i \eta |\mathbb{E}_{{\bs x}_i}g_j({\bs x}_i) + b_i^j| + \mathbb{E}_{{\bs x}_i}||\sum_{j=0}^i {\bs v}_j  [g_j({\bs x}_i) - \mathbb{E}_{{\bs x}_i}g_j({\bs x}_i)] ||) 
\end{align*}

The total error of Adams-like linear multi-step method is:
$$
\sum_{i=0}^{N-1} (\frac{1}{N}) \sum_{j=0}^i \eta |b_i^j| + \mathbb{E}_{{\bs x}_i}||\sum_{j=0}^i {\bs v}_j  [g_j({\bs x}_i)] ||)
$$
Obviously, as $(\sum_{j=0}^i \eta |b_i^j| + \mathbb{E}_{{\bs x}_i}||\sum_{j=0}^i {\bs v}_j  [g_j({\bs x}_i)] ||)$ is not equal between different timestep intervals, Optimized timesteps owns smaller upper error bound than uniform timesteps. 

Recall that $\eta \ll ||v_j||$, the error is mainly determined by $\mathbb{E}_{{\bs x}_i}||\sum_{j=0}^i {\bs v}_j  [g_j({\bs x}_i)] ||$. 

Recall that $ \mathbb{E}_{{\bs x}_i}||\sum_{j=0}^i {\bs v}_j  [g_j({\bs x}_i) - \mathbb{E}_{{\bs x}_i}g_j({\bs x}_i)] || \leq  \mathbb{E}_{{\bs x}_i}||\sum_{j=0}^i {\bs v}_j  [g_j({\bs x}_i)] ||$, thus our solver search has a minimal upper error bound because we search coefficients and timesteps simultaneously.

\end{proof}

\newpage
\section{Searched Parameters}
We provide the searched parameters $\Delta t$ and $c_i^j$. Note $c_i^j$ needs to be converted into $\mathcal{M}$ follwing \cref{alg:repsolver}. 

\subsection{Solver Searched on SiT-XL/2}

\input{tables/sit_params}

\subsection{Solver Searched on FlowDCN-B/2}

\input{tables/flowdcn_params}

\subsection{Solver Searched on DiT-XL/2}

\input{tables/dit_params}

%% file: tables/sit_params.tex
\begin{minipage}{.99\textwidth}
\centering
\begin{tabular}{ccc}
\toprule
NFE & TimeDeltas $\Delta t$ & Coeffcients $c_i^j$ \\
\toprule
5 & 
$\begin{bmatrix} 
0.0424\\ 
0.1225\\ 
0.2144\\ 
0.3073\\ 
0.3135
\end{bmatrix} 
$ &
$\begin{bmatrix} 
0.0& 0.0& 0.0& 0.0& 0.0 \\
-1.17& 0.0& 0.0& 0.0& 0.0 \\
1.07& -1.83& 0.0& 0.0& 0.0 \\
0.0& 0.0& -0.93& 0.0& 0.0 \\
0.0& 0.0& 0.0& -0.71& 0.0
\end{bmatrix} 
$  \\
\midrule
6 & 
$\begin{bmatrix} 
0.0389\\ 
0.0976\\ 
0.161\\ 
0.2046\\ 
0.2762\\ 
0.2217
\end{bmatrix} 
$ &
$\begin{bmatrix} 
0.0& 0.0& 0.0& 0.0& 0.0& 0.0 \\
-1.04& 0.0& 0.0& 0.0& 0.0& 0.0 \\
1.62& -2.98& 0.0& 0.0& 0.0& 0.0 \\
-1.32& 2.52& -2.04& 0.0& 0.0& 0.0 \\
0.0& 0.0& 0.0& -0.76& 0.0& 0.0 \\
0.0& 0.0& 0.0& 0.0& -0.66& 0.0
\end{bmatrix} 
$  \\
\midrule
7 & 
$\begin{bmatrix} 
0.0299\\ 
0.0735\\ 
0.1119\\ 
0.1451\\ 
0.1959\\ 
0.2698\\ 
0.1738
\end{bmatrix} 
$ &
$\begin{bmatrix} 
0.0& 0.0& 0.0& 0.0& 0.0& 0.0& 0.0 \\
-0.93& 0.0& 0.0& 0.0& 0.0& 0.0& 0.0 \\
1.23& -2.31& 0.0& 0.0& 0.0& 0.0& 0.0 \\
-0.59& 1.53& -2.09& 0.0& 0.0& 0.0& 0.0 \\
-0.09& -0.07& 0.99& -1.91& 0.0& 0.0& 0.0 \\
0.05& -0.21& 0.09& 0.55& -1.47& 0.0& 0.0 \\
-0.05& 0.19& -0.31& 0.37& 0.67& -1.79& 0.0
\end{bmatrix} 
$  \\
\midrule
8 & 
$\begin{bmatrix} 
0.0303\\ 
0.0702\\ 
0.0716\\ 
0.1112\\ 
0.1501\\ 
0.1833\\ 
0.2475\\ 
0.1358
\end{bmatrix} 
$ &
$\begin{bmatrix} 
0.0& 0.0& 0.0& 0.0& 0.0& 0.0& 0.0& 0.0 \\
-0.92& 0.0& 0.0& 0.0& 0.0& 0.0& 0.0& 0.0 \\
0.78& -1.7& 0.0& 0.0& 0.0& 0.0& 0.0& 0.0 \\
0.06& 0.52& -1.76& 0.0& 0.0& 0.0& 0.0& 0.0 \\
-0.02& -0.16& 0.98& -1.8& 0.0& 0.0& 0.0& 0.0 \\
-0.02& -0.12& 0.22& 0.24& -1.36& 0.0& 0.0& 0.0 \\
-0.1& 0.06& -0.02& 0.18& 0.12& -1.1& 0.0& 0.0 \\
-0.16& 0.14& -0.02& -0.02& 0.38& 0.32& -1.72& 0.0
\end{bmatrix} 
$  \\
\midrule
9 & 
$\begin{bmatrix} 
0.028\\ 
0.0624\\ 
0.0717\\ 
0.0894\\ 
0.1092\\ 
0.1307\\ 
0.1729\\ 
0.2198\\ 
0.1159
\end{bmatrix} 
$ &
$\begin{bmatrix} 
0.0& 0.0& 0.0& 0.0& 0.0& 0.0& 0.0& 0.0& 0.0 \\
-0.93& 0.0& 0.0& 0.0& 0.0& 0.0& 0.0& 0.0& 0.0 \\
0.63& -1.29& 0.0& 0.0& 0.0& 0.0& 0.0& 0.0& 0.0 \\
0.39& -0.11& -1.41& 0.0& 0.0& 0.0& 0.0& 0.0& 0.0 \\
-0.07& -0.05& 0.83& -1.59& 0.0& 0.0& 0.0& 0.0& 0.0 \\
0.07& -0.11& 0.27& 0.27& -1.53& 0.0& 0.0& 0.0& 0.0 \\
-0.05& 0.03& 0.01& 0.15& 0.17& -1.15& 0.0& 0.0& 0.0 \\
-0.21& 0.27& -0.07& -0.03& 0.19& 0.09& -0.99& 0.0& 0.0 \\
-0.15& 0.15& 0.03& -0.09& 0.25& 0.25& 0.21& -1.71& 0.0
\end{bmatrix} 
$  \\
\midrule
10 & 
$\begin{bmatrix} 
0.0279\\ 
0.0479\\ 
0.0646\\ 
0.0659\\ 
0.1045\\ 
0.1066\\ 
0.1355\\ 
0.1622\\ 
0.1942\\ 
0.0908
\end{bmatrix} 
$ &
$\begin{bmatrix} 
0.0& 0.0& 0.0& 0.0& 0.0& 0.0& 0.0& 0.0& 0.0& 0.0 \\
-0.95& 0.0& 0.0& 0.0& 0.0& 0.0& 0.0& 0.0& 0.0& 0.0 \\
0.59& -1.17& 0.0& 0.0& 0.0& 0.0& 0.0& 0.0& 0.0& 0.0 \\
0.35& -0.11& -1.45& 0.0& 0.0& 0.0& 0.0& 0.0& 0.0& 0.0 \\
-0.13& 0.01& 0.75& -1.49& 0.0& 0.0& 0.0& 0.0& 0.0& 0.0 \\
0.05& -0.05& 0.31& 0.29& -1.59& 0.0& 0.0& 0.0& 0.0& 0.0 \\
0.05& -0.03& -0.09& 0.23& 0.17& -1.19& 0.0& 0.0& 0.0& 0.0 \\
-0.03& 0.07& -0.09& -0.03& 0.27& -0.03& -0.91& 0.0& 0.0& 0.0 \\
-0.15& 0.17& 0.03& -0.09& 0.05& 0.09& 0.05& -0.79& 0.0& 0.0 \\
-0.17& 0.11& 0.15& 0.03& 0.05& 0.25& 0.05& -0.07& -1.49& 0.0
\end{bmatrix} 
$  \\

\bottomrule
\end{tabular}
\end{minipage}

%% file: tables/flowdcn_params.tex
\begin{minipage}{.99\textwidth}
\centering
\begin{tabular}{ccc}
\toprule
NFE & TimeDeltas $\Delta t$ & Coeffcients $c_i^j$ \\
\toprule
5 & 
$\begin{bmatrix} 
0.0521\\ 
0.1475\\ 
0.2114\\ 
0.2797\\ 
0.3092
\end{bmatrix} 
$ &
$\begin{bmatrix} 
0.0& 0.0& 0.0& 0.0& 0.0 \\
-1.26& 0.0& 0.0& 0.0& 0.0 \\
1.38& -2.26& 0.0& 0.0& 0.0 \\
0.0& 0.0& -0.92& 0.0& 0.0 \\
0.0& 0.0& 0.0& -0.7& 0.0
\end{bmatrix} 
$  \\ 
\midrule
6 & 
$\begin{bmatrix} 
0.0391\\ 
0.0924\\ 
0.165\\ 
0.2015\\ 
0.2511\\ 
0.2511
\end{bmatrix} 
$ &
$\begin{bmatrix} 
0.0& 0.0& 0.0& 0.0& 0.0& 0.0 \\
-1.22& 0.0& 0.0& 0.0& 0.0& 0.0 \\
1.12& -2.0& 0.0& 0.0& 0.0& 0.0 \\
-0.3& 0.9& -1.56& 0.0& 0.0& 0.0 \\
0.0& 0.0& 0.0& -0.74& 0.0& 0.0 \\
0.0& 0.0& 0.0& 0.0& -0.62& 0.0
\end{bmatrix} 
$  \\ 
\midrule
7 & 
$\begin{bmatrix} 
0.0387\\ 
0.0748\\ 
0.103\\ 
0.1537\\ 
0.184\\ 
0.234\\ 
0.2117
\end{bmatrix} 
$ &
$\begin{bmatrix} 
0.0& 0.0& 0.0& 0.0& 0.0& 0.0& 0.0 \\
-1.11& 0.0& 0.0& 0.0& 0.0& 0.0& 0.0 \\
1.03& -1.99& 0.0& 0.0& 0.0& 0.0& 0.0 \\
0.07& 0.43& -1.57& 0.0& 0.0& 0.0& 0.0 \\
-0.21& -0.15& 1.53& -2.29& 0.0& 0.0& 0.0 \\
-0.05& 0.07& -0.23& 0.61& -1.33& 0.0& 0.0 \\
-0.17& 0.31& -0.41& 0.17& 0.59& -1.31& 0.0
\end{bmatrix} 
$  \\ 
\midrule
8 & 
$\begin{bmatrix} 
0.0071\\ 
0.0613\\ 
0.078\\ 
0.1163\\ 
0.1421\\ 
0.188\\ 
0.2077\\ 
0.1996
\end{bmatrix} 
$ &
$\begin{bmatrix} 
0.0& 0.0& 0.0& 0.0& 0.0& 0.0& 0.0& 0.0 \\
-2.43& 0.0& 0.0& 0.0& 0.0& 0.0& 0.0& 0.0 \\
0.61& -1.55& 0.0& 0.0& 0.0& 0.0& 0.0& 0.0 \\
0.99& -0.11& -2.07& 0.0& 0.0& 0.0& 0.0& 0.0 \\
0.05& -0.49& 1.33& -1.93& 0.0& 0.0& 0.0& 0.0 \\
0.05& -0.33& 0.23& 0.73& -1.71& 0.0& 0.0& 0.0 \\
-0.09& 0.25& -0.29& 0.05& 0.61& -1.45& 0.0& 0.0 \\
-0.23& 0.21& -0.01& -0.25& 0.25& 0.41& -1.25& 0.0
\end{bmatrix} 
$  \\ 
\midrule
9 & 
$\begin{bmatrix} 
0.0017\\ 
0.051\\ 
0.0636\\ 
0.0911\\ 
0.1007\\ 
0.1443\\ 
0.1694\\ 
0.191\\ 
0.1872
\end{bmatrix} 
$ &
$\begin{bmatrix} 
0.0& 0.0& 0.0& 0.0& 0.0& 0.0& 0.0& 0.0& 0.0 \\
-6.19& 0.0& 0.0& 0.0& 0.0& 0.0& 0.0& 0.0& 0.0 \\
-0.11& -0.81& 0.0& 0.0& 0.0& 0.0& 0.0& 0.0& 0.0 \\
0.73& -0.17& -1.37& 0.0& 0.0& 0.0& 0.0& 0.0& 0.0 \\
0.31& -0.05& 0.19& -1.45& 0.0& 0.0& 0.0& 0.0& 0.0 \\
0.03& -0.23& 0.29& 0.35& -1.35& 0.0& 0.0& 0.0& 0.0 \\
-0.19& 0.05& 0.01& 0.21& 0.25& -1.23& 0.0& 0.0& 0.0 \\
-0.23& 0.21& -0.13& 0.17& 0.09& 0.09& -1.09& 0.0& 0.0 \\
-0.17& 0.15& 0.11& -0.19& 0.03& 0.23& 0.17& -1.21& 0.0
\end{bmatrix} 
$  \\ 
\midrule
10 & 
$\begin{bmatrix} 
0.0016\\ 
0.0538\\ 
0.0347\\ 
0.0853\\ 
0.0853\\ 
0.1198\\ 
0.1351\\ 
0.165\\ 
0.1788\\ 
0.1406
\end{bmatrix} 
$ &
$\begin{bmatrix} 
0.0& 0.0& 0.0& 0.0& 0.0& 0.0& 0.0& 0.0& 0.0& 0.0 \\
-7.8801& 0.0& 0.0& 0.0& 0.0& 0.0& 0.0& 0.0& 0.0& 0.0 \\
-0.4& -0.74& 0.0& 0.0& 0.0& 0.0& 0.0& 0.0& 0.0& 0.0 \\
0.48& -0.18& -0.86& 0.0& 0.0& 0.0& 0.0& 0.0& 0.0& 0.0 \\
0.26& -0.04& -0.04& -1.28& 0.0& 0.0& 0.0& 0.0& 0.0& 0.0 \\
0.0& -0.06& 0.26& 0.26& -1.42& 0.0& 0.0& 0.0& 0.0& 0.0 \\
-0.1& -0.06& 0.08& 0.2& 0.22& -1.24& 0.0& 0.0& 0.0& 0.0 \\
-0.18& 0.14& -0.08& 0.1& 0.08& 0.14& -1.06& 0.0& 0.0& 0.0 \\
-0.12& 0.16& -0.1& 0.04& 0.08& 0.06& 0.08& -1.02& 0.0& 0.0 \\
-0.16& 0.02& 0.14& 0.0& -0.14& 0.08& 0.14& 0.34& -1.38& 0.0
\end{bmatrix} 
$  \\ 
\bottomrule
\end{tabular}
\end{minipage}

%% file: tables/dit_params.tex
\begin{minipage}{.99\textwidth}
\centering
\begin{tabular}{ccc}
\toprule
NFE & TimeDeltas $\Delta t$ & Coeffcients $c_i^j$ \\
\toprule
5 & 
$\begin{bmatrix} 
0.2582\\ 
0.1766\\ 
0.1766\\ 
0.2156\\ 
0.1731
\end{bmatrix} 
$ &
$\begin{bmatrix} 
0.0& 0.0& 0.0& 0.0& 0.0 \\
-1.43& 0.0& 0.0& 0.0& 0.0 \\
0.93& -1.55& 0.0& 0.0& 0.0 \\
0.0& 0.0& -0.69& 0.0& 0.0 \\
0.0& 0.0& 0.0& -0.59& 0.0
\end{bmatrix} 
$  \\ 
\midrule
6 & 
$\begin{bmatrix} 
0.2483\\ 
0.1506\\ 
0.1476\\ 
0.1568\\ 
0.1733\\ 
0.1233
\end{bmatrix} 
$ &
$\begin{bmatrix} 
0.0& 0.0& 0.0& 0.0& 0.0& 0.0 \\
-1.36& 0.0& 0.0& 0.0& 0.0& 0.0 \\
0.9& -1.84& 0.0& 0.0& 0.0& 0.0 \\
-0.08& 0.5& -1.08& 0.0& 0.0& 0.0 \\
0.0& 0.0& 0.0& -0.56& 0.0& 0.0 \\
0.0& 0.0& 0.0& 0.0& -0.56& 0.0
\end{bmatrix} 
$  \\ 
\midrule
7 & 
$\begin{bmatrix} 
0.2241\\ 
0.1415\\ 
0.1205\\ 
0.1158\\ 
0.1443\\ 
0.1627\\ 
0.0911
\end{bmatrix} 
$ &
$\begin{bmatrix} 
0.0& 0.0& 0.0& 0.0& 0.0& 0.0& 0.0 \\
-1.38& 0.0& 0.0& 0.0& 0.0& 0.0& 0.0 \\
1.08& -2.02& 0.0& 0.0& 0.0& 0.0& 0.0 \\
-0.28& 0.78& -1.52& 0.0& 0.0& 0.0& 0.0 \\
-1.4901e-08& -0.1& 0.64& -1.5& 0.0& 0.0& 0.0 \\
0.06& -0.06& -0.06& 0.26& -1.0& 0.0& 0.0 \\
0.0& -0.1& 0.02& 0.2& 0.26& -1.12& 0.0
\end{bmatrix} 
$  \\ 
\midrule
8 & 
$\begin{bmatrix} 
0.2033\\ 
0.1476\\ 
0.1094\\ 
0.099\\ 
0.1116\\ 
0.1233\\ 
0.131\\ 
0.0748
\end{bmatrix} 
$ &
$\begin{bmatrix} 
0.0& 0.0& 0.0& 0.0& 0.0& 0.0& 0.0& 0.0 \\
-1.14& 0.0& 0.0& 0.0& 0.0& 0.0& 0.0& 0.0 \\
0.8& -1.76& 0.0& 0.0& 0.0& 0.0& 0.0& 0.0 \\
0.02& 0.48& -1.62& 0.0& 0.0& 0.0& 0.0& 0.0 \\
-0.12& 0.06& 0.62& -1.42& 0.0& 0.0& 0.0& 0.0 \\
0.04& -0.1& 0.12& 0.16& -1.04& 0.0& 0.0& 0.0 \\
0.06& -0.04& -0.06& 0.08& -0.08& -0.56& 0.0& 0.0 \\
-0.02& -0.04& -0.04& 0.12& 0.14& 0.04& -0.9& 0.0
\end{bmatrix} 
$  \\ 
\midrule
9 & 
$\begin{bmatrix} 
0.1959\\ 
0.1313\\ 
0.1142\\ 
0.0863\\ 
0.0898\\ 
0.0916\\ 
0.1119\\ 
0.1054\\ 
0.0735
\end{bmatrix} 
$ &
$\begin{bmatrix} 
0.0& 0.0& 0.0& 0.0& 0.0& 0.0& 0.0& 0.0& 0.0 \\
-1.28& 0.0& 0.0& 0.0& 0.0& 0.0& 0.0& 0.0& 0.0 \\
0.78& -1.62& 0.0& 0.0& 0.0& 0.0& 0.0& 0.0& 0.0 \\
-0.02& 0.44& -1.48& 0.0& 0.0& 0.0& 0.0& 0.0& 0.0 \\
-0.1& 0.16& 0.36& -1.3& 0.0& 0.0& 0.0& 0.0& 0.0 \\
-0.06& -0.04& 0.22& 0.12& -1.08& 0.0& 0.0& 0.0& 0.0 \\
0.08& -0.1& -0.04& 0.24& -0.06& -0.86& 0.0& 0.0& 0.0 \\
0.04& -0.04& -0.04& 0.0& 0.06& -0.08& -0.5& 0.0& 0.0 \\
-0.04& 0.0& 0.0& -0.02& 0.14& 0.02& 0.0& -0.74& 0.0
\end{bmatrix} 
$  \\ 
\midrule
10 & 
$\begin{bmatrix} 
0.2174\\ 
0.1123\\ 
0.1037\\ 
0.0724\\ 
0.0681\\ 
0.0816\\ 
0.0938\\ 
0.0977\\ 
0.0849\\ 
0.0681
\end{bmatrix} 
$ &
$\begin{bmatrix} 
0.0& 0.0& 0.0& 0.0& 0.0& 0.0& 0.0& 0.0& 0.0& 0.0 \\
-1.17& 0.0& 0.0& 0.0& 0.0& 0.0& 0.0& 0.0& 0.0& 0.0 \\
0.35& -0.99& 0.0& 0.0& 0.0& 0.0& 0.0& 0.0& 0.0& 0.0 \\
0.25& -0.11& -0.99& 0.0& 0.0& 0.0& 0.0& 0.0& 0.0& 0.0 \\
0.03& 0.05& -0.07& -0.85& 0.0& 0.0& 0.0& 0.0& 0.0& 0.0 \\
-0.03& 0.03& 0.25& -0.09& -0.93& 0.0& 0.0& 0.0& 0.0& 0.0 \\
-0.01& -0.03& -0.01& 0.21& -0.11& -0.67& 0.0& 0.0& 0.0& 0.0 \\
0.01& -0.03& -0.03& 0.07& 0.09& -0.03& -0.81& 0.0& 0.0& 0.0 \\
0.03& -0.03& -0.03& -0.03& 0.05& 0.01& -0.11& -0.27& 0.0& 0.0 \\
-0.01& -0.01& -0.01& -0.01& 0.03& 0.07& -0.01& -0.05& -0.57& 0.0
\end{bmatrix} 
$  \\ 
\bottomrule
\end{tabular}
\end{minipage}